\documentclass[lettersize,journal]{IEEEtran}
\usepackage{amsmath,amsfonts}
\usepackage{algorithmic}
\usepackage{array}
\usepackage[caption=false,font=normalsize,labelfont=sf,textfont=sf]{subfig}
\usepackage{textcomp}
\usepackage{stfloats}
\usepackage{url}
\usepackage{verbatim}
\usepackage{graphicx}
\def\BibTeX{{\rm B\kern-.05em{\sc i\kern-.025em b}\kern-.08em
    T\kern-.1667em\lower.7ex\hbox{E}\kern-.125emX}}
\usepackage{balance}
\usepackage{cite}
\usepackage{hyperref}
\usepackage{cleveref}
\usepackage{amsthm}

\usepackage{hyperref}
\usepackage{xcolor}

\usepackage{algorithmic}

\RequirePackage{algorithm}
\RequirePackage{algorithmic}

\newtheorem{remark}{Remark}
\newtheorem{theorem}{Theorem}
\newtheorem{definition}{Definition}
\newtheorem{corollary}{Corollary}[theorem]

\usepackage{xcolor}
\newcommand{\rebuttalnew}[1]{#1}
\newcommand{\rebuttal}[1]{#1}

\begin{document}
\title{Edge Delayed Deep Deterministic Policy Gradient:\\ efficient continuous control for edge scenarios}


\author{Alberto Sinigaglia$^1$, Niccolò Turcato$^2$, Carli Ruggero$^2$, Gian Antonio Susto$^{1,2}$

\thanks{This study was partially carried out within the PNRR research activities of the consortium iNEST (Interconnected North-Est Innovation Ecosystem) funded by the European Union Next-GenerationEU (Piano Nazionale di Ripresa e Resilienza (PNRR) - Missione 4 Componente 2, Investimento 1.5 - D.D. 1058 23/06/2022, ECS00000043).}

\thanks{$^1$ Human Inspired Technology Research Centre, University of Padova, Padova (PD), Italy}
\thanks{
$^2$Department of Information Engineering, University of Padova, Padova (PD), Italy}
}

\markboth{Journal of IEEE Transactions on Automation Science and Engineering}{Authors: Edge Delayed Deep Deterministic Policy Gradient: efficient continuous control for edge scenarios}

\maketitle

\begin{abstract}
\rebuttal{Deep Reinforcement Learning (DRL) has emerged as a powerful paradigm for learning complex policies directly from high-dimensional input spaces, enabling advances across a variety of domains. Modern DRL algorithms often rely on dual-network Q-learning architectures to approximate optimal policies to overcome overestimation bias. Recent research has introduced approaches leveraging multiple Q-functions to further mitigate overestimation effects and enhance policy reliability. However, there is a growing emphasis on deploying DRL in edge scenarios, where privacy concerns and stringent hardware constraints necessitate highly efficient algorithms. In such environments, the computational and memory efficiency of learning methods is of critical importance. In this context, we propose Edge Delayed Deep Deterministic Policy Gradient (EdgeD3), a novel reinforcement learning algorithm specifically designed for edge computing settings. EdgeD3 offers significant reductions in GPU time (by 25\%) and computational and memory usage (by 30\%), while consistently achieving or surpassing the performance of state-of-the-art algorithms across multiple benchmarks and in real-world tasks.}
\end{abstract}

\renewcommand{\abstractname}{Note to Practitioners}
\begin{abstract}
Driven by the growing need for efficient computational solutions in automation, this research introduces the Edge Delayed Deep Deterministic Policy Gradient (EdgeD3), a novel Deep Reinforcement Learning algorithm designed for applications on edge devices with a limited computational budget.
EdgeD3 is developed to enable on-device execution of policy learning, when computing resources are at a premium, such as in smart manufacturing and autonomous vehicle systems.
EdgeD3 can learn highly effective policies, on par of state-of-the-art algorithms, while utilizing significantly fewer resources.
This would empower autonomous devices to operate independently of cloud-based systems, fostering faster operational speeds and enhancing data privacy.
\end{abstract}


\section{Introduction}
\label{sec:intro}
\IEEEPARstart{A}{utonomous} agents operating in dynamic environments require robust training methods, with Reinforcement Learning (RL) offering a potent framework for such continuous adaptation. Control strategies are of paramount importance in domains characterized by continuous action spaces, as discussed in recent literature \cite{recht2019continuous_control}. 

\rebuttal{Such approaches have been applied to nagivation \cite{cai2024deep,hao2021deep,yan2021hybrid}, robotic manipulation \cite{cui2024task,zhou2025sample} and Athletic Intelligence \cite{wiebe2024reinforcement}. Furthermore, they are also gaining attention in the multi-agent settings, controlling independently multiple agents \cite{miao2024effective,li2025transfer,li2024multi}.}

Actor-critic methods, particularly those utilizing temporal difference learning \cite{dann2014policy_temporal_diff}, represent a foundational approach within this area. However, the integration of Q-learning with deep neural networks to form a critic function has set new benchmarks \cite{watkins1992qlearning}, enhancing policy optimization through advanced policy gradient techniques \cite{lillicrap2015ddpg, fujimoto2018td3, haarnoja2018sac}.  \\
These state-of-the-art approaches, however, are not without their pitfalls, particularly the tendency of Q-learning to induce overestimation bias \cite{van2016deep}. To combat this, the Twin Delayed Deep Deterministic Policy Gradient (TD3) method was developed, innovatively applying Clipped Double Q-Learning (CDQ) to refine accuracy in action value estimations \cite{fujimoto2018td3}. By deploying two independently trained neural networks, this method systematically lowers the risk of overestimating the Q-values by adopting the lesser of the two during the training process. Despite its benefits, CDQ can introduce underestimation bias, albeit with a lesser impact on policy modifications when compared to overestimation, as argued by the proponents of TD3 \cite{fujimoto2018td3}. Furthermore, the introduced estimator leaves very little room for tuning the trade-off.  Indeed, new developments \cite{lyu2022efficient} built on top of TD3 optimize for a convex combination of the minimum between the estimates, the CDQ estimate, and the maximum of the two. \rebuttal{Thanks to such improvements, these algorithms have been widely applied to various sectors, such as robotic grasping \cite{zhou2024t,aslam2025dartbot}, wind turbine control \cite{egbomwan2024physics}, and path following \cite{ma2023sample}.}

\rebuttal{Recently, new algorithms further increased the pool of Q-networks in order to have a less noisy estimate of the true Q-learning target \cite{li2023realistic, chen2020randomized,shu2025episodic,messaoud2024s, sturm2025deep}, using up $10$ or more networks}. These new enlarged pools of independent estimates are usually combined with multiple steps of updates, overshadowing previous algorithms' performances, while TD3\cite{fujimoto2018td3} performs a single update per step. These enhancements introduce a very noticeable overhead, using up to $10$x more memory and $10$x more computational resources than the original algorithm they build on top of.

The increase in computational cost of these new algorithms \cite{li2023realistic, chen2020randomized,shu2025episodic,messaoud2024s, sturm2025deep} is prohibitive in many low-resource scenarios, for example, in edge computing applications. Many edge computing applications can benefit from data-driven approaches. For example, in the domain of autonomous driving, real-time sensing and decision processing can be conducted by deploying edge computing nodes on self-driving vehicles; this advancement enables a reduction in response times and enhances driving safety \cite{liu2020computing}. Similarly, in the field of smart healthcare, edge computing nodes utilized on wearable and medical devices allow for the monitoring of patients’ physiological parameters in real-time. The collected data is then transmitted to the cloud, where it is analyzed and used for diagnosis, facilitating the realization of telemedicine and personalized medicine \cite{catarinucci2015iot}. The ability to deploy deep learning models on the edge is lately gaining more and more attention due to the appealing properties of on-device learning, for its decentralized computation, for economic purposes, and for its privacy-preserving nature. Such a problem can be tackled using quantization \cite{wang2021enabling} or pruning \cite{cheng2023survey}. However, for the case of Reinforcement Learning, in order to allow the personalization of such models to the context in which the edge device is deployed, it is also important to tackle the algorithmic side. Indeed, differently from most other areas of machine learning, Deep RL requires additional computation and additional models for the learning to be carried out, thus adding a non-negligible overhead. \rebuttal{Many applications \cite{cai2024deep,hao2021deep,yan2021hybrid,cui2024task,zhou2025sample,miao2024effective,li2025transfer,li2024multi} can be found using these approaches to learn effective policies, and notably most of them share the property that a real world deployment of the proposed approach would imply a constrained computational reosurce availability. For this reason, it is of paramount importance to consider the computational cost of Deep RL algorithms, even over raw policy-performance, which is an aspect drastically underrated by these latest advancements in continuous control.}

In this paper, we propose an alternative to Deep Deterministic Policy Gradient (DDPG) that tackles overestimation with a single $Q$ estimate. This algorithm, called Edge Delayed Deep Deterministic Policy Gradient (EdgeD3), has a computational cost lower than DDPG by $25\%$ while maintaining the same memory footprint. It does so with a new expectile loss, which induces an underestimation bias that evens the overestimation caused by Q-learning. Even though being cheaper, EdgeD3 has performances comparable or superior to the state-of-the-art in most cases, while using $30\%$ less processing resources and having $30\%$ less memory footprint compared to such more advanced methods. Thanks to its memory efficiency and requiring less computational resources, it is much more suited for low-resource settings, such as edge computing applications, where CPU computing time, energy savings, and memory usage are highly impactful. Furthermore, such an ability to require less computing and memory can allow for in-device learning, preserving privacy, which is ever more important in such scenarios. In addition, the introduced loss formulation allows for more control over the estimation bias with respect to the CDQ approach. To benchmark and compare the proposed algorithm \rebuttal{in simulation}, we use a selection of Mujoco \cite{mujoco} robotics environments from the OpenAI Gym suite \cite{OpenAI_gym}. \rebuttal{Furthermore, we also validate the proposed approach on real-world navigation tasks, using custom TurtleBots, which are widely used mobile robots \cite{xin2021multimobile,anderson2020mobile,gao2023efficient}, with limited computing resources.}

\rebuttal{The key contributions of this paper are summarized as follows:}

\begin{itemize}
    \item \rebuttal{We propose a novel approach to tackle overestimation in DDPG-like algorithms using Expectile estimations, thus not requiring additional components;}
    \item \rebuttal{We show that such a quantity can be efficiently estimated via samples, thus allowing convergence under simulation-based learning;}
    \item \rebuttal{We empirically evaluate the proposed approach on different robotic simulation tasks to estimate the improvement brought by the novel loss formulation;}
    \item \rebuttal{We test the proposed approach on two different real-world navigation tasks to highlight the importance of update frequency.}
\end{itemize}

The paper is structured as follows: in Section \ref{sec:related_work}, we discuss recent contributions that are relevant to the scope of this paper, while in Section \ref{sec:background}, we briefly review the required theoretical background. \Cref{sec:overest_underest} discusses the newly introduced loss and its theoretical foundation. In Section \ref{sec:smoothing} we discuss the stabilization of the optimization procedure for the DDPG algorithm and we introduce the final EdgeD3 algorithm. \Cref{sec:experiments} compares the proposed EdgeD3 algorithm to the state-of-the-art both from a computational expensiveness side and performance point-of-view, \rebuttal{both on simulation and on two real-world navigation tasks}. Finally, Section \ref{sec:conclusions} concludes the paper by briefly discussing some potential future lines of contributions.

\section{Related work}
\label{sec:related_work}
The problem of estimation bias in Value function approximation has been widely recognized and addressed in numerous studies. Notably, Q-learning has been identified to exhibit overestimation bias in discrete action spaces, as highlighted in \cite{thrun2014issues_funapprox}. A seminal response to this challenge was Double Q-learning, introduced by Van Hasselt \cite{hasselt2010double_qlearning}, marking a foundational development in this field. Building on this, the Maxmin Q-learning approach \cite{lan2019maxmin_qlearning} demonstrated that employing a broader ensemble of more than two Q estimates could substantially alleviate this bias and enhance the efficacy of Q-learning.

In the context of controlling continuous action spaces, recent contributions have tackled the dual issues of underestimation and overestimation biases by ensembling Q function estimates \cite{kuznetsov2020controlling_overest, chen2020randomized,wei2022controlling,li2023realistic,shu2025episodic,messaoud2024s, sturm2025deep}. Truncated Quantile Critics (TQC) was proposed by Kuznetsov et al. \cite{kuznetsov2020controlling_overest}, which extends the Soft Actor-Critic (SAC) by integrating an ensemble of five critic estimates. This model not only offers a distributional representation of the critic but also implements truncation during critic updates to reduce overestimation bias. Moreover, TQC refines the actor update mechanism by employing an averaged ensemble of Qs.

Mirroring some aspects of TQC, Randomized Ensembled Double Q-learning (REDQ) developed by Chen et al. \cite{chen2020randomized} uses a larger ensemble of 10 networks. Unlike TQC, In REDQ, the critic estimates are updated multiple times for each step in the environment, 20 times in the presented results.

Further advancing the methodology, Quasi-Median Q-learning (QMQ) introduced by Wei et al. \cite{wei2022controlling} employs four Q estimates and the quasi-median operator to compute the targets for critic updates. This method highlights a trade-off approach between overestimation and underestimation, while the policy gradient is computed relative to the mean of the Qs.

Additionally, the Realistic Actor-Critic (RAC) approach by Li et al. \cite{li2023realistic} seeks to strike a balance between value overestimation and underestimation, employing an ensemble of 10 Q networks. In this approach, the ensemble of Q functions is updated 20 times per environmental step using targets computed from the mean of the Qs minus one standard deviation, with the actor update maximizing the mean of the Q functions. This ensemble strategy has been applied successfully to both TD3 and SAC, achieving competitive performance and sample efficiency akin to Model-Based RL \cite{janner2019mbpo}.

\rebuttal{Recent advancements \cite{shu2025episodic,messaoud2024s,sturm2025deep} have introduced innovative methods for aggregating Q-value predictions to achieve more accurate and unbiased, non-overestimated evaluations of action quality. In contrast, \cite{messaoud2024s} proposes a distinct methodology that leverages an energy-based model (EBM) to address the overestimation problem. However, this approach necessitates the use of sampling procedures to obtain the final estimate, rendering it impractical for deployment in resource-constrained environments.}

\rebuttal{\emph{\textbf{Remark:}}} the contributions discussed in this section enhance TD3 or SAC by employing large ensembles of $Q$ Networks in the critic or by performing multiple steps of training at each time step, thereby \rebuttal{making them infeasible for edge scenarios}. In contrast, our work concentrates on enhancing the performance by tackling overestimation bias, \rebuttal{further improving over} the CDQ mechanism, \rebuttal{while reducing} computational burden of the DDPG algorithm.

\section{Background}
\label{sec:background}

Reinforcement Learning (RL) is modeled as a Markov Decision Process (MDP), encapsulated by the tuple $(\mathcal{S}, \mathcal{A}, \mathcal{P}, R, \gamma)$. The components $\mathcal{S}$ and $\mathcal{A}$ represent the continuous state and action spaces, respectively, necessitating a continuous transition density function $\mathcal{P}: \mathcal{S} \times \mathcal{A} \times \mathcal{S} \rightarrow [0, \infty)$. The reward function is denoted as $r: \mathcal{S} \times \mathcal{A} \times \mathcal{S} \rightarrow \mathbb{R}$, and the discount factor is given by $\gamma \in [0,1]$. Policy $\mu$, parameterized by $\phi$, is a mapping at each time-step $t$ from the current state $s_t \in \mathcal{S}$ to an action $a_t \in \mathcal{A}$, defined by the conditional distribution $\mu_\phi(a_t|s_t)$. This policy is often realized through a Neural Network.

The principal objective in RL is to find a policy $\mu$ that maximizes the expected discounted sum of rewards, mathematically expressed as $R_0 = \mathbb{E}_{\mu_\phi}[\sum_{t=0}^\infty \gamma^t r_t]$. To facilitate this, RL utilizes two primary constructs: the value function $V$ and the action-value function $Q$, defined as:
\begin{align}
V^\mu(s) &= \mathbb{E}_{\mu}[R_t \mid S_t = s] \\
Q^\mu(s, a) &= \mathbb{E}_{\mu}[R_t \mid S_t = s, A_t = a]
\end{align}

Within the framework of Reinforcement Learning, the $Q$ function can be recursively defined, leading to its essential role in both theoretical exploration and practical application:
\begin{equation}
Q^\mu(s, a) = \mathbb{E}_{\mu}[r(s,a,s') + \gamma \mathbb{E}[Q^\mu(s', a')]]
\end{equation}
Under the assumption that the policy at subsequent time $t+1$ is optimal, Q-learning reformulates the $Q$ function as:
\begin{equation}
Q^\mu(s, a) = \mathbb{E}[r(s,a,s') + \gamma \max_a Q^\mu(s', a)]
\end{equation}
This off-policy characterization allows the $Q$ function to depend solely on the environmental dynamics. If the greedy policy is modeled as a neural network $\mu$, parameterized by $\phi$, the expression for the $Q$ function refines to:
\begin{equation}
Q^\mu(s, a) = \mathbb{E}[r(s,a,s') + \gamma Q^\mu(s', \mu_\phi(s'))]
\end{equation}

The Deterministic Policy Gradient (DPG) algorithm \cite{silver2014deterministic_policy_grad} employs the $Q$ function to derive a policy update rule for a differentiable model $Q_\theta$:
\begin{align}
    \nabla J(\phi) &= \nabla_\phi [Q_\theta(s, a)|_{s = s_t, a = \mu_\phi(s)}], \\
    &= \nabla_a [Q_\theta(s, a)|_{s = s_t, a = \mu_\phi(s)}] \nabla_\phi \mu_\phi(s) |_{s=s_t}
\end{align}
Thanks to this reformulation, we are able to train concurrently a Q-function and a policy using gradient descent without relying on Policy gradient techniques, which are outperformed by the former.

Building on top of DDPG, TD3 \cite{fujimoto2018td3} introduces strategies to mitigate the overestimation bias prevalent in prior models. It incorporates a double $Q$ estimation to temper the learning targets:
\begin{equation}
    y = r + \gamma \min_{i = 1, 2} Q_{\theta'_i}(s',a'), \text{where } a' \sim \mu_{\phi'}(s')
\end{equation}
Additionally, TD3 employs two Exponential Moving Averages (EMA) of the networks, which are updated at each learning step, enhancing the stability. 
\section{Overestimation and Underestimation in Q-Learning}
\label{sec:overest_underest}
When dealing with discrete action spaces, the Value function can be optimized with Q-learning with the greedy target $y = r + \text{ max}_{a'}Q(s', a')$. However, in \cite{thrun2014issues_funapprox}, it has been proven that if this target has an error, then the maximum over the value biased by this error will be greater than the true maximum in expectation. Consequently, even when errors initially have mean zero, they probably lead to consistent overestimation biases in the updates of values, which are then carried through the Bellman equation. In \cite{fujimoto2018td3}, the authors have shown both analytically and experimentally that this overestimation bias is also present in actor-critic methods.
While the overestimation may seem minor with each value update, the authors express two concerns. First, if not addressed, the overestimation could accumulate into a more substantial bias over numerous updates. Second, an imprecise value estimate has the potential to result in suboptimal policy updates. This poses a significant issue, as it initiates a feedback loop where suboptimal actions, favored by the inaccurate critic, can be reinforced in subsequent policy updates. For these reasons, CDQ was introduced in \cite{fujimoto2018td3} in the TD3 algorithm, showing significant improvements with respect to previous state-of-the-art, i.e., DDPG. However, CDQ has two main drawbacks: (i) it introduces an uncontrollable underestimation bias in the critic, and (ii) memory and computation consumption are doubled in the critic estimate due to the introduction of a second $Q$ network. This expensiveness is shared by SAC, too, and is exacerbated in newer, improved alternatives, as discussed previously. 

The rest of this section proposes an extension to DDPG for the control of the overestimation bias with an alternative strategy to CDQ without the need to introduce a second $Q$ network.

\subsection{Tackling Overestimation with a single  $Q$ estimate}
\label{subsec:reducing_q_overest_single_q_est}
TD3 applies CDQ in the critic updates in order to favor underestimation over overestimation, hoping to counterbalance the bias introduced by Q-learning. Even though TD3 is an effective algorithm and theoretically sound, taking the minimum between the two estimates leaves very little room for adjusting this bias in case we have any evidence that it's hurting the performances. For this reason, we explore a method that allows more control over a possible underestimation bias to compensate for the overestimation induced by Q-learning, with a single $Q$ function estimate, thus making it computationally cheaper and having a smaller memory footprint compared to the CDQ mechanism shared by TD3 and SAC. Specifically, we propose to change the CDQ mechanism with an Expectile Regression Loss for a single $Q$ function.

The $\tau$ expectile in probability theory for a cumulative density function $F$ of the random variable $X$ is the solution of the following equation \cite{newey1987asymmetric}:
\begin{equation}
\label{eq:expectile_orig}
(1-\tau) \int_{-\infty}^t (t-x) dF(x)  = \tau \int_t^{+\infty} (x-t) dF(x)
\end{equation}
However, the value for $\tau$ can be hard to interpret, so, for this reason, we will use the following equivalent definition using two hyperparameters $\alpha, \beta$:
\begin{align}
\label{eq:expectile_alpha_beta}
\begin{split}
\frac{\alpha}{2\max(\alpha, \beta)} \int_{-\infty}^t& (t-x) dF(x)  = \\
&\frac{\beta}{2\max(\alpha, \beta)} \int_t^{+\infty} (x-t) dF(x)
\end{split}
\end{align}

Indeed, it can be seen how if we set $\tau = 0.5$, then $t = \mathbb{E}[X]$. More specifically, $\tau$ defines a monotonically increasing mapping with respect to $t$, thus allowing to control the distance to the mean.

\begin{definition}
    We say that a function $f: \mathbb{R} \rightarrow \mathbb{R}$ is monotonic non decreasing if and only if, given $x_1,x_2\in \mathbb{R}$ and $x_1 < x_2$, then $f(x_1) \le f(x_2)$
\end{definition}

\begin{theorem}[Monotonicity of Expectiles]
\label{thm:monotonic_expectiles}
    The function defined in \cref{eq:expectile_orig} is monotonic non-decreasing, thus, given $\tau_1 \le \tau_2$, then $t_1 \le t_2$, with $t_1$ and $t_2$ the respective expectiles solution of \cref{eq:expectile_orig} $\tau_1 \le \tau_2$.
\end{theorem}

\begin{proof}
    We first need to consider that the \Cref{eq:expectile_orig} is actually a function. Such function is defined as:
    $$
    f(\tau) = t \text{ s.t. } (1-\tau) \int_{-\infty}^t (t-x) dF(x)  = \tau \int_t^{+\infty} (x-t) dF(x)
    $$
    We need to show that such a function is monotonic non-decreasing; thus, if $\tau$ increases, then the corresponding $f(\tau)$ cannot decrease.
    
    Dividing both sides by $(1-\tau)$ and then by $\int_t^{+\infty} (x-t) dF(x)$, we can rewrite \cref{eq:expectile_orig} in the following way:
    $$
    \frac{\tau}{1-\tau} = \frac{\int_{-\infty}^t (t-x) dF(x)}{\int_t^{+\infty} (x-t) dF(x)}
    $$
    We can observe that the left-hand side is monotonic with respect to $\tau$. Indeed its derivative is $g(\tau) ' = \frac{1}{(1-\tau)^2} \ge 0 \,\forall \tau \in (0,1)$. Furthermore, we can observe that the integrands on the right-hand side are non-negative. Considering that:
    \begin{align*}
    \text{given } f(x) &\ge 0, a \le \min(b_1,b_2): \\
    &\int_a^{b_1} f(x)dx \le \int_a^{b_2} f(x)dx \iff b_1 \le b_2,\\
    \text{given } f(x) &\ge 0, \max(a_1,a_2) \le b: \\
    &\int_{a_1}^{b} f(x)dx \le \int_{a_2}^{b} f(x)dx \iff a_1 \le a_2,
    \end{align*}
    we can thus conclude that the right-hand side is monotonic with respect to $t$.

    To conclude, since the left-hand side is monotonic with respect to $\tau$, that the right-hand side is monotonic with respect to $t$, and that the equality between the two sides has to be preserved, then we can conclude that $t$ is monotonic with respect to $\tau$ and vice versa.

\end{proof}

The same holds true for the formulation in \cref{eq:expectile_alpha_beta}. In particular, given two hyperparameters $\alpha \in \mathbb{R^+}, \beta \in \mathbb{R^+}$ Expectile Regression is the solution of an asymmetric loss, in particular, the Mean Squared Loss that relaxes one of the two sides of the function.
\begin{equation}
\label{eq:expectile}
    L^{\alpha, \beta}(f_\theta(x), y) = \frac{1}{Z}\begin{cases}
\alpha\, (y - f_\theta(x))^2  & \text{if $f_\theta(x) < y$} \\
\beta\, (y - f_\theta(x))^2 & \text{otherwise} 
\end{cases},
\end{equation}
with $Z = \max(\alpha, \beta)$, \rebuttal{$y$ being the target and $f_\theta(x)$ being the regressor network}. 
Thanks to \cref{thm:monotonic_expectiles}, we can prove that $\alpha$ and $\beta$ control the overestimation-underestimation bias, namely, \rebuttal{denoting by $f^{\alpha, \beta}(x)$ the optimal solution to $L^{\alpha, \beta}(f(x), y)$}:
\begin{enumerate}
    \item $\alpha = \beta$ reverts to Mean Squared Error (MSE), as the solution of $L^{c, c}$ for any $c > 0$ is exactly the MSE.
    \item $\alpha < \beta$ favors underestimation errors, as $f^{\alpha, \beta}(x) \le f^{c,c}(x)$, with $f^{\alpha, \beta}(x)$ solution of $L^{\alpha, \beta}$ and $f^{c,c}(x)$ solution of $L^{c, c}$ for any $c > 0$
    \item $\alpha > \beta$ favors overestimation errors, as $f^{\alpha, \beta}(x) \ge f^{c,c}(x)$, with $f^{\alpha, \beta}(x)$ solution of $L^{\alpha, \beta}$ and $f^{c,c}(x)$ solution of $L^{c, c}$ for any $c > 0$
\end{enumerate}
In \cref{fig:expreg}, we show the function learned by a fourth-degree polynomial minimizing the expectile loss varying $\alpha, \beta$ approximating the function $y = 0.1 x^3 + \epsilon, \,\, \epsilon \sim N(0, 8)$. Features were sampled uniformly from $U(-10, 10)$, and both features and targets were normalized in the range $\left[0,1\right]$. The parameters are initialized to $0$ and optimized using Adam\cite{kingma2014adam} with $0.001$ stepsize. Different $\alpha,\beta$ tradeoffs learn different expectiles of the distribution $p(y|x)$.
\begin{figure}[h!]
    \centering
    \vspace{-0.5cm}
    \includegraphics[width=0.8\columnwidth]{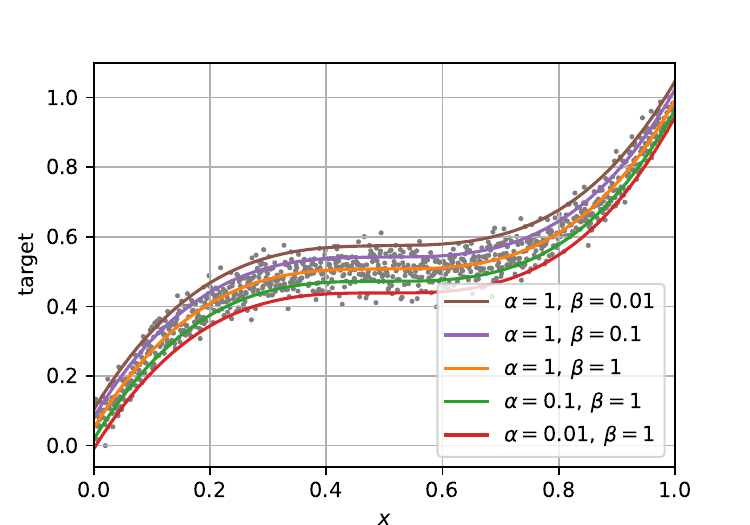}
    \vspace{-0.5cm}
    \caption{Third-degree polynomials learned optimizing the expectile loss for different values for $\alpha$ and $\beta$. }
    \vspace{-0.5cm}
    \label{fig:expreg}
\end{figure}

It can be noted that using two hyper-parameters $\alpha, \beta$ for the loss leads to an overparametrization. Indeed, $\tau$ was enough to parametrize it. Furthermore, since we are normalizing by $\max(\alpha, \beta)$, the two hyper-parameters are scale-invariant, and thus we can just consider $\alpha, \beta \in [0,1]$. However, for the sake of readability, we will keep the two separated even though we could have used a single one.\

The expectile loss then leads to the following objective for the DDPG algorithm for the optimization of the $Q$ function:
\begin{equation}
\label{eq:expectile_objective}
L(\theta) = \mathbb{E}[L^{\alpha, \beta}(Q_\theta(s,a), r(s,a) + \gamma Q_{\theta'}(s', \mu_{\phi'}(s')))].
\end{equation}

As mentioned earlier, we add a normalizing constant $Z$ in front of the equation in order to have a fair comparison between algorithms. From an optimization standpoint, it's equivalent to a change in the step size of the optimizer. Indeed, thanks to $Z$, the type of error we prefer to penalize, the one with the highest coefficient, has exactly $1$ as a constant in front, leading to an update that is equivalent to the original method. For the other type of error, on the other hand, the loss is multiplied by a constant $<1$, leading to a lower step size. This way, we can guarantee that the improvements shown by this proposal are due to the effectiveness of the loss and not by bigger step sizes induced by hidden constants in the loss.

\begin{corollary}[Asymptotic convergence of Q-learning with Expectile loss]
\label{cor:decay}
    Given a scheduled decay function $\lambda(t)$, such that $\lim_{t\rightarrow +\infty} \lambda(t) = 1$, we can directly reduce the induced bias by the expectile loss, by applying at each step:
    \begin{equation}
    \min(\alpha^{t+1}, \beta^{t+1}) \leftarrow \min(\alpha^t, \beta^t) + |\alpha^t - \beta^t| \cdot \lambda(t) .
    \end{equation}

    This ensures that eventually, no bias will be introduced by the expectile loss, as $\lim_{t \rightarrow +\infty} L^{\alpha, \beta}(f(x), y) = L^{1,1}(f(x), y)$, which is the original Bellman Optimality Equation update \cite{sutton2018reinforcement}.
\end{corollary}

Since Q-learning is guaranteed to converge starting from any policy that has support over all actions part of the optimal policy, and it's assumed to have infinite time, since DDPG algorithms during training add as exploration noise $\epsilon \sim N(0, \sigma^2_{exp})$, it's trivial to see how such decay allows for a theoretically sound convergence.

\begin{remark}[Consistency and Asymptotic Normality of Empirical Expectiles]
\label{thm:asymptotic_expectiles}
Let \(\{X_i\}_{i=1}^n\) be i.i.d.\ samples from a distribution with cumulative distribution function \(F\), and let \(t_n\) be the empirical \(\tau\)-expectile obtained by minimizing the sample counterpart of Equation \ref{eq:expectile_orig}. Then, under standard regularity conditions for M-estimators, \(t_n\) is a consistent estimator of the true \(\tau\)-expectile \(t^*\) (i.e., \(t_n \xrightarrow{p} t^*\)), and it satisfies a central limit theorem:
\[
\sqrt{n}(t_n - t^*) \xrightarrow{d} \mathcal{N}\left(0,\; \left[\Psi(t^*, \tau)\right]^{-2}\right),
\]
where \(\Psi(t, \tau)\) is the derivative of the population objective function (\cref{eq:expectile_orig}) associated with the expectile definition at \(t\).
\end{remark}

\begin{corollary}[Monte Carlo Estimation of Expectile Function]
\label{cor:montecarlo}
When the integrals in Equation \ref{eq:expectile_orig} are approximated via Monte Carlo sampling,
\[
\hat{t} = \arg\min_{t} \sum_{i=1}^n \omega_i \cdot \ell_\tau(t, X_i),
\]
where \(\ell_\tau(\cdot,\cdot)\) represents the asymmetric loss defining the expectile, and \(\omega_i\) are weights (often uniform, \(\omega_i = \frac{1}{n}\)), then \(\hat{t}\) inherits the consistency and asymptotic normality as described in Theorem \ref{thm:asymptotic_expectiles}. Consequently,
\[
\sqrt{n}(\hat{t} - t^*) \xrightarrow{d} \mathcal{N}\left(0,\; \left[\Psi(t^*, \tau)\right]^{-2}\right).
\]
\end{corollary}

For more information and formal proofs of \cref{thm:asymptotic_expectiles} and \cref{cor:montecarlo} refer to \cite{newey1987asymmetric,knight1998limiting}.

\begin{remark}
These theoretical results validate the use of the expectile loss in place of traditional MSE-based updates by confirming that the estimator converges to the correct level of under- or over-estimation specified by the parameters \(\alpha\) and \(\beta\) (or equivalently \(\tau\)). These properties also justify the robustness and adaptability of expectile-based methods in function approximation and value estimation within reinforcement learning frameworks.
\end{remark}

To conclude, \cref{thm:monotonic_expectiles} shows how the proposed loss can directly manipulate the wanted overestimation/underestimation, while \cref{cor:decay} show how to prove the convergence of the proposed method, ensuring no bias will be introduced at convergence. Finally, \cref{thm:asymptotic_expectiles} and \cref{cor:montecarlo} show that the expectile loss \cref{eq:expectile} allows a consistent and asymptotic convergence to the true expectile, allowing \cref{thm:monotonic_expectiles} to hold also in a setting where Monte Carlo estimation is employed.

\begin{figure*}[h!]
    \centering
    \includegraphics[width=\textwidth]{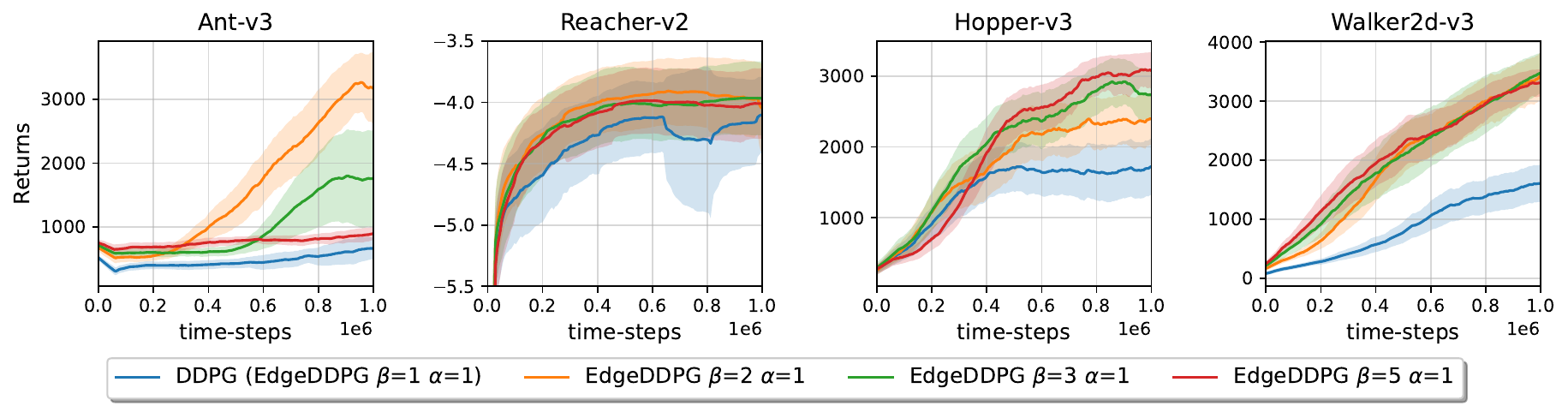}
    \vspace{-0.5cm}
    \caption{Training progress curves for continuous control tasks in OpenAI Gym, showing the effect of different choices of $\alpha, \beta$ in EdgeDDPG. Plots and shaded areas indicate mean and standard deviation, respectively from evaluation across 10 trials. Benchmarks were performed on 10 random seeds for simulator and network initializations. Curves are smoothed uniformly for visual clarity. }
    \vspace{-0.5cm}
    \label{fig:edgeddpg}
\end{figure*}
To demonstrate the effectiveness of this novel loss formulation, we equip DDPG with this new objective. We refer to this new updated version of the algorithm as Edge Deep Deterministic Policy Gradient (EdgeDDPG). This new algorithm shares the same memory footprint as DDPG, and adds a small additional cost to the update step due to the need to check whether the residuals $Q(\mu(s),a) - (r + \gamma Q(s',\mu(s')))$ are positive or negative, to then apply the correct corresponding loss. However, the difference in computational cost is almost indistinguishable from our tests.\\
A complete description of the algorithm can be found in \cref{sec:code-expddpg}. The new objective \cref{eq:expectile_objective} lays the ground for the final and proposed algorithm, EdgeD3, which will be discussed in \Cref{sec:smoothing}.\\
It can be seen how the modification, since its simplicity, consists in a small change in the definition of the DDPG algorithm. However, nonetheless, the resulting performance gains are noteworthy. Indeed, even though the two algorithms have the same computational and memory footprint, it can be seen from \cref{fig:edgeddpg} how such small modification allows the algorithm to improve in all tasks, to the point of going from non-converging in \textit{Ant} at all, to an actual policy learning.\\
\Cref{fig:edgeddpg} furthermore shows the tradeoff between overestimation and underestimation, highlighting how both of them can be detrimental. Indeed, it shows how preferring underestimation over overestimation has some diminishing returns eventually, so pushing for a strong underestimation bias can be as detrimental as accepting the overestimation bias of Q-learning. In \cref{sec:app_ablation_trade_off}, we report further tests of the algorithm on additional tasks and with different choices of $\alpha$ and $\beta$.\\
Even though it is computationally cheaper and brings very substantial performance improvements, a quick crosscheck between results in this paper will show how this new simple modification to the DDPG still struggles to reach state-of-the-art performances.

\section{Smoothing the optimization landscape}
\label{sec:smoothing}
Actor-critic RL algorithms for continuous control algorithms are mainly composed of 2 components: the Q-function $Q_\theta$ and the policy network $\mu_\phi$. The 
optimization criterion of the latter one, can be seen as gradient ascent procedures on the former one. However, differently from the usual optimization settings where we assume the function to be stationary throughout the optimization procedure\cite{ruder2016overview}, in DDPG-like algorithms $Q_\theta$ changes over time.
Due to this property, and due to the fact that is a conditional optimization, as it depends on the state $s$, it's very hard to take advantage of new optimization techniques \cite{kingma2014adam}, such as adaptive stepsizes and momentum. Furthermore, recently, adversarial attacks \cite{goodfellow2014explaining} have shown how the landscape of a neural network is far from being smooth, and that small changes in the input, such as the action computed by $\mu_\phi$ fed to $Q_\theta(s,a)$, can lead to big changes in the output, in our case in the Q-estimate.
For this reason, is very important to tackle this ill-conditioning of the Q-network $Q_\theta$ as also addressed in \cite{fujimoto2018td3, nachum2018smoothed}. One such way, applied in GANs, is to build and regularize the final network in such a way that it is almost 1-Lipschitz \cite{miyato2018spectral}. Another way, also used in GANs, is to penalize the gradient. In order to apply such a method in the DDPG case would lead to the following loss:
\begin{equation}
L(\theta) = (r + \gamma Q_{\theta'}(s', \mu_{\phi'}(s')) - Q_\theta(s,a))^2 + \xi ||\nabla_a Q(s,a)||^2
\end{equation}
Even though this formalization can work, it is very computationally expensive, as it requires first estimating the gradient of the Q-estimate with respect to the action, computing the norm of it, and then doing the full loss gradient update. However, \cite{rifai2011contractive} shows how penalizing the gradient of the output, in our case the Q-estimate, to the input, in our case the action, is equivalent to adding noise.
Even though the two are theoretically similar, they have very different computational costs. With this technique, in our case, this implies solving the following optimization objective:
\begin{equation}
L(\theta) = (\mathbb{E}_{\epsilon \sim p(x)}[r + \gamma Q_{\theta'}(s', \mu_{\phi'} + \epsilon(s'))] - Q_\theta(s,a))^2.
\end{equation}
This new formulation, applied to the EdgeDDPG algorithm, equates to solving the expectation over the Expectile loss:
\begin{equation}
\label{eq:noisyexploss}
L(\theta) = L^{\alpha, \beta}(Q_\theta(s,a), \mathbb{E}_{\epsilon \sim p(x)}[r + \gamma Q_{\theta'}(s', \mu_{\phi'} + \epsilon(s'))]).
\end{equation}
Indeed, $p(x)$ wants to be a smoothing function, having $y = \int p(\epsilon) [r + \gamma Q_{\theta'}(s', \mu_{\phi'} + \epsilon(s'))] d\epsilon$ so it's suggested to choose a distribution $p$ centered in $0$ and symmetric, in order to avoid introducing any bias.

In addition to the smoothing, for the optimization to be effective, we need $Q_\theta(s, \mu_\phi(s)) \approx Q^*(s, \mu_\phi(s))$ in a neighborhood of $\mu_\phi(s)$, so that \\ $\nabla_{\mu_\phi(s)} Q_\theta(s, \mu_\phi(s)) \approx \nabla_{\mu_\phi(s)} Q^*(s, \mu_\phi(s))$. For this reason, we can take advantage of the delayed update introduced in \cite{fujimoto2018td3}. This way, not only do we offer more time to the $Q_\theta$ to improve, but we also save computational resources, skipping $k-1$ updates every $k$, where $k$ is the frequency of actor updates.

Combining EdgeDDPG, the delayed update from \cite{fujimoto2018td3}, and the noisy estimate from \cref{eq:noisyexploss}, we obtain Edge Delayed Deep Deterministic Policy Gradient.

A complete description of EdgeD3 can be found in \cref{alg:edged3}. Indeed, it is a small modification on the proposed EdgeDDPG algorithm with almost no additional costs: the delayed policy update saves computation, and the noisy update, in our case, is estimated with One-Sample Monte Carlo (OSMC)\cite{song2021train}. Thus, the only additional cost is brought by the noise generation, which is almost negligible.

\begin{algorithm}[h!]
   \caption{Edge Delayed Deep Deterministic Policy Gradient (EdgeD3)}
   \label{alg:edged3}
\begin{algorithmic}[5]
    \STATE Given $\alpha, \beta, \tau_1, \tau_2$ and $\lambda(t)$
   \STATE Initialize critic $Q_{\theta}$, and actor $\mu_\phi$ networks
   \STATE Initialize target networks $\theta' \leftarrow \theta, \phi' \leftarrow \phi$
   \STATE Initialize $k$ for the actor update frequency
   \STATE Initialize $p(x)$ for the target input noise
   \STATE Initialize replay memory \textbf{$\mathcal{B}$}
   \STATE Initialize $t=0$
   \REPEAT
       \REPEAT
            \STATE $t \leftarrow t + 1$
            \STATE Select action with exploration noise $a \sim \mu_{\phi'}(s) + \omega$, $\omega \sim \mathcal{N}(0, \sigma)$ and observe $r$ and  $s'$.
            \STATE $d=\begin{cases} 1 \text{ if } s'\text{ is terminal} \\ 0 \text{ otherwise} \end{cases}$
            \STATE Store $(s,a,r,s',d)$ tuple in $\mathcal{B}$
            \STATE Sample batch of $N$ tuples $(s,a,r,s',d)$ from $\mathcal{B}$
            \STATE $y = \mathbb{E}_{\epsilon \sim p(x)}[r + \gamma Q_{\theta'}(s', \epsilon + \mu_{\phi'}(s')) \cdot (1-d)]$ 
            \STATE $\nabla L(\theta) = \nabla_\theta N^{-1} \Sigma L^{\alpha,\beta}(y, Q_{\theta}(s,a))$ [\cref{eq:expectile}]
            \STATE Update $Q_\theta$ via GD using $\nabla L(\theta)$
            \IF{$t$ mod $k$}
                \STATE Update $\phi$ by deterministic policy gradient:
                \STATE $\nabla_\phi J(\phi) = \frac{1}{N} \Sigma \nabla_a Q_{\theta}(s,a) |_{a=\mu_\phi(s)} \nabla_\phi \mu_{\phi}(s)$
                \STATE Update target networks:
                \STATE $\theta' \leftarrow \tau_1 \theta + (1-\tau_1) \theta'$
                \STATE $\phi' \leftarrow \tau_2 \phi + (1-\tau_2) \phi'$
            \ENDIF
            
       \UNTIL{$d$ is false}
       \STATE $\min(\alpha, \beta) \leftarrow \min(\alpha, \beta) + |\alpha - \beta| \cdot \lambda(t) $
   \UNTIL{$t < T$}
\end{algorithmic}
\end{algorithm}

In Fig. \ref{fig:edgeddpg_v_edged3}, we compare the EdgeDDPG algorithm with the EdgeD3 algorithm.

\begin{figure*}[h!]
    \centering
    \includegraphics[width=\textwidth]{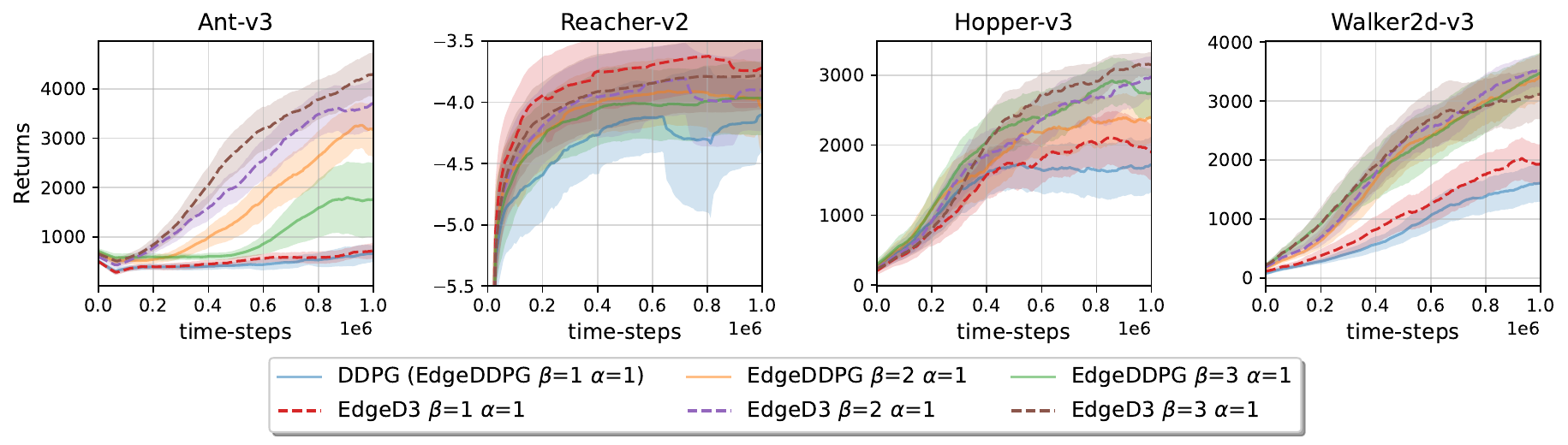}
    \caption{Training progress curves for continuous control tasks in OpenAI Gym, showing the effect of different choices of $\alpha, \beta$ in EdgeD3 compared to EdgeDDPG. Plots and shaded areas indicate mean and standard deviation, respectively, from evaluation across 10 trials. Benchmarks were performed on 10 random seeds for simulator and network initializations. Curves are smoothed uniformly for visual clarity. }
    \label{fig:edgeddpg_v_edged3}
\end{figure*}
\section{Experiments}
\label{sec:experiments}
To evaluate our proposed algorithm, we benchmark its performance on the suite of Mujoco \cite{mujoco}, a set of robotic environments aimed for continuous control, with no change to the environment itself or the reward to improve reproducibility. \\ 
For our implementation of all the algorithms, we used a feed-forward network composed of 3 layers of 256 neurons, optimized using Adam optimizer \cite{kingma2014adam} with $3\cdot 10^{-4}$ stepsize for a fair comparison. For EdgeD3, TD3 and SAC we considered $k=2$, thus the actor is updated every $2$ updates of the Q-functions.
For the target smoothing distribution, we use the proposed clipped Gaussian distribution also used in \cite{fujimoto2018td3}, and for the exploration policy, we used a Gaussian distribution $N(0, 0.1)$ for all the algorithms, apart from SAC \cite{haarnoja2018sac} where we used the learned posterior distribution. More technical details for reproducibility can be found in \cref{sec:app_reproducibility}. In the rest of this section, we firstly compare the memory usage of the proposed algorithm with state of the art, secondly, their GPU-time utilization, and thirdly, we compare learning performance on the Mujoco benchmarks. \rebuttal{All results repoted come from networks trained using \cref{eq:expectile} with $\alpha=1, \beta=2$, thus showing that very limited hyperparameter tuning is needed, and that a performant initial guess exists.}

\subsection{Resource use comparison}
\label{sec:resource_use}
The proposed algorithm EdgeD3 aims at being a step towards RL-algorithms that are suited for Edge Computing, which lately is gaining a lot of attention thanks to its natural ability to be scalable and highly privacy-preserving, as all the computation is done on-device. Such a setting, however, requires the use of the least amount of computational resources as well as memory resources. Indeed, edge computing algorithms aim at achieving the following characteristics\cite{carvalho2021edge, deng2020edge, shi2016edge}:
\begin{itemize}
    \item minimal CPU usage: the processing power of an edge device is limited in order to keep the cost of the device low;
    \item minimal memory usage: as per the CPU usage, the memory is also limited for production cost;
    \item minimal computation: many edge devices, such as smartphones, are powered by a battery, and having CPU-intensive algorithms leads to shorter battery duration but also shorter overall life of the battery due to overheating.
\end{itemize}
The comparisons have been carried out on a computer equipped with an AMD Ryzen Threadripper 1920X 12-Core Processor, an NVIDIA Titan V with 12Gb of memory and 128Gb of RAM, running Ubuntu $18.04.6$. During the comparison, all unnecessary processes were properly killed, the update routine was paused, and no other major process was running.\\

\subsubsection{RAM usage}
The most popular algorithms that build upon DDPG, implementing the DPG estimator\cite{silver2014deterministic_policy_grad}, are SAC and TD3, which both utilize the CDQ mechanism, exploiting an ensemble of two Q functions to estimate the Q-learning target.
However, having this additional function also implies having another additional target function, thus effectively having four networks in total to maintain in memory.
Such additional cost is justified by the improvements in performance. However, if memory is a concern, such as for edge computing, this additional cost might not be worth it.
For this reason, we will compare the algorithms by their memory consumption. This section, in conjunction with \cref{sec:experiments}, shows we can achieve state-of-the-art performances with much less memory required, thanks to the new loss formulation.

Since all algorithms share the same Replay Buffer size, the only factor influencing the footprint is the number of networks that the algorithms require. 
In \cref{table:memcomp}, we show the percentage of decrease in peak memory usage of 10.000 update steps of each algorithm using 10-dimensional fake Gaussian noise data generated at the beginning, thus removing the replay buffer from the memory consumption, and also removing the environment, which might cause some sharp increases in memory usage biasing the estimates. The test was carried out using the CPU for the computation so that the memory used by the process was not split across RAM and GPU memory. 

\begin{table}[h]
\vspace{-0.5cm}
\caption{Comparison of percentage of peak additional memory used compared to EdgeD3.}
\vspace{-0.5cm}
\label{table:memcomp}
\vskip 0.1in
\begin{center}
\begin{small}
\begin{sc}
\begin{tabular}{|c|c|}
\hline
\textbf{Algorithm} & \textbf{\% of RAM used compared to ExpD3} \\
\hline\hline
DDPG  & $-1.2\%$ \\
TD3   & $+29.3\%$ \\
SAC   & $+31.1\%$ \\
\hline
\end{tabular}
\end{sc}
\end{small}
\end{center}
\vspace{-0.5cm}
\end{table}
\begin{figure*}[h!]
    \centering
    \includegraphics[width=\textwidth]{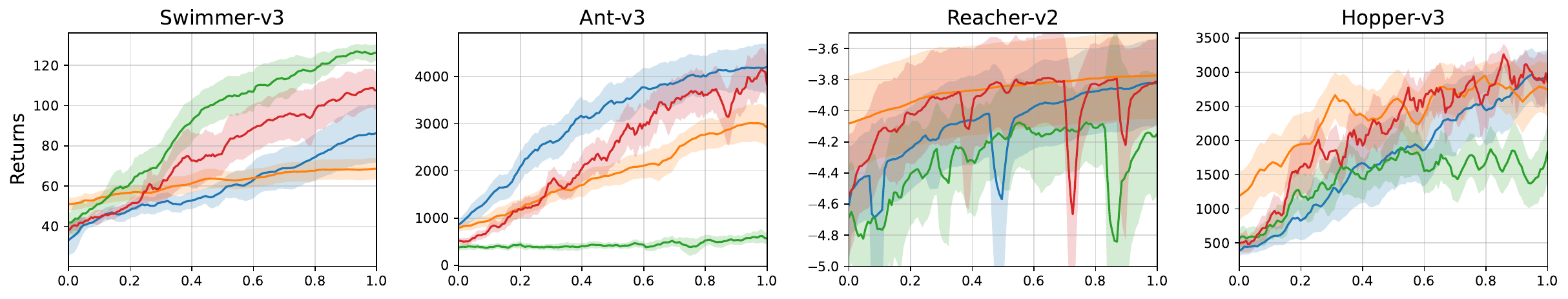}
    
    \includegraphics[width=0.75\textwidth]{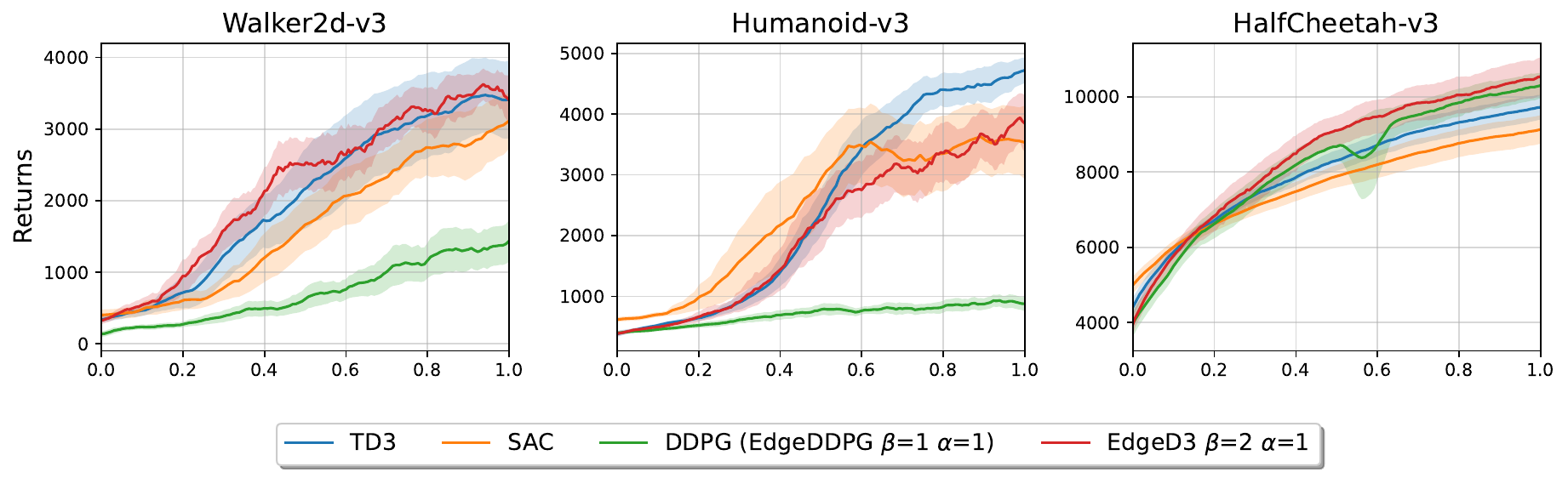}
    \caption{
    Comparing EdgeD3 with baselines in continuous control tasks. 
    Plots are from 10 random seeds for simulator and network initializations, smoothed for visualization. Evaluations of Return are performed every 5000 time steps; plots show a standard deviation over 10 episodes.}
    \label{fig:benchmarks_edged3}
\end{figure*}

\subsubsection{GPU-time comparison}
TD3 and SAC are popular algorithms that build upon DDPG. They both exploit an ensemble of two Q estimates for the calculation of the Q-learning target.
During training however, we are required to train them independently, and we need to pay the computational expenses of the additional Q-function.
Instead, EdgeD3 requires the same network as the original DDPG algorithm, avoiding such additional computational costs.

In order to have a better sense of the improvement brought by EdgeD3 for edge computing from a computational point of view, we compared the different algorithms to showcase the various footprints.
The hyper-parameters used for the comparison are the default ones used in the respective papers. The only exception to this is for Soft Actor-Critic, which originally did not use the delayed update. For a fair comparison, we report both the computational time of SAC with the delayed update (SAC$^d$) and without (SAC$^o$). For the benchmarks, however, we report the original SAC implementation, which updates the actor at every step, together with the Q-functions.

Since all have a Replay Buffer $\mathcal{B}$, and that all at inference time have roughly the same cost, as the state is forwarded through the policy network $\mu_\phi$ and some Gaussian noise is added to the network prediction, the only component that can vary the computational cost, is the training loop. For this reason, we create random data from a $10$-dimensional Gaussian distribution and use that as fake states for the forward passes. We then proceed to run $10000$ training steps for algorithms and repeat it over $10$ different seeds.
The results of the comparisons are reported in \cref{fig:gpucomp}. For clarity, we also report the improvements in \cref{table:gpucomp}.
Indeed, it's evident how the proposed method, compared to the other methods, requires 30\% less computing time since either they lack the delayed update or they require the update of 2 Q-functions.

\begin{table}[h]
\caption{Comparison of GPU-time. Details on the hardware can be found in \cref{sec:resource_use}. Between parenthesis is reported the percentage of time saved by EdgeD3 compared to the various methods.}
\label{table:gpucomp}
\vskip 0.15in
\begin{center}
\begin{small}
\begin{sc}
\begin{tabular}{|c|c|}
\hline
\textbf{Algorithm} & \textbf{Time}  \\
\hline\hline
EdgeD3 & $214.0\pm 7.1$ms  \\
DDPG   & $285.5 \pm 7.4$ms ($-25.0\%$) \\
TD3    & $308.2 \pm 2.7$ms ($-30.5\%$) \\
SAC$^d$    & $320.90 \pm 3.6$ms ($-33.3\%$) \\
SAC$^o$    & $492.91 \pm 2.9$ms ($-56.8\%$) \\
\hline
\end{tabular}
\end{sc}
\end{small}
\end{center}
\end{table}

\begin{figure}[t]
    \centering
    \includegraphics[width=0.8\columnwidth]{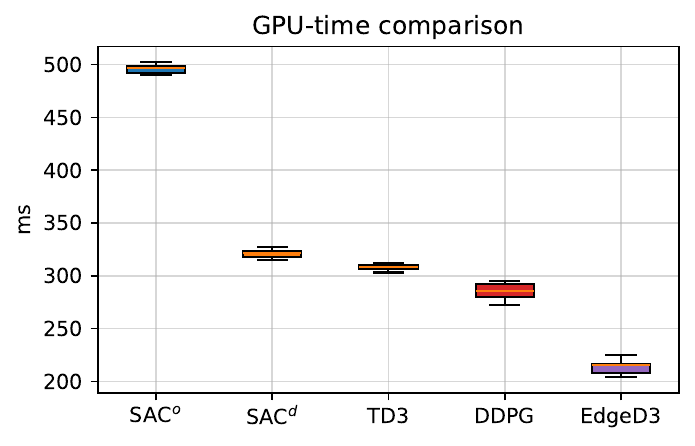}
    \caption{GPU time for one step of training loop for the different algorithms averaged over $10000$ steps, repeated over 10 different seeds. Details on the hardware can be found in \cref{sec:resource_use}}
    \label{fig:gpucomp}
\end{figure}

However, it has to be noted that an important aspect cannot be captured by the memory and GPU usage tests carried out. Edge devices have very limited memory, and part of the computation required to make an algorithm work is to have the networks loaded in memory. However, for the sake of the experiments, we have assumed that all the networks could be loaded in memory all at once, which is not guaranteed in an edge scenario. For this reason, the reported results should be considered as a best-case scenario comparison.

\subsection{\rebuttal{Comparison to state-of-the-art}}
For the comparison with other algorithms, as reported at the beginning of the section, we will use the Mujoco suit. For reproducibility, we used the same criterion used by the authors TD3 \cite{fujimoto2018td3}, having a bootstrap phase at the beginning of the learning. Each task ran for $1$ million steps and was evaluated every $5000$ step on 10 different environment initializations per evaluation. The reported results are also averaged over 10 different independent learning with a different seed each for different environments and network initializations.

We compare our proposed algorithm to TD3 \cite{fujimoto2018td3}, SAC \cite{haarnoja2018sac}, and the original DDPG algorithm, following the small tweaks proposed by the authors of TD3 \cite{fujimoto2018td3} for their comparison.

However, in order to have a fair comparison, instead of comparing the algorithms based on the environment steps, we compare them based on their training time. For this reason, since it's the cheapest across the pool, we first run the EdgeD3 algorithm in all tasks, keeping track of the wall-clock time required to do so for each environment. Then, we proceed to run all the other algorithms for the same amount of time in order to take into consideration the additional computational cost brought by the double Q-function update shown in \cref{table:gpucomp}.

The final evaluations are reported in \cref{fig:benchmarks_edged3}. In \cref{table:finalcomp}, we report the best score achieved by the various algorithms, allowing each the same amount of wall-clock time. In bold, we report the best algorithms, whereas the underlined ones are the two best.

It can be seen how EdgeD3 is consistently part of the two best algorithms. Furthermore, it matches and even surpasses performances of state-of-the-art methods while having a significantly smaller memory footprint, as reported in \cref{table:memcomp}. Indeed, it can also be seen that is not perfect, as it still struggles on very complex tasks such as Humanoid. On the other hand, its memory-wise competitor, DDPG, is highly outperformed by EdgeD3 in almost all the tasks, even going from a non-convergence regime to a state-of-the-art policy.

\subsection{\rebuttal{Real world edge evaluation}}
\label{sec:real_world_edge}

\rebuttal{To further validate the applicability of EdgeD3 beyond simulated environments, we conducted a set of real-world edge evaluations on a custom-built TurtleBot platform equipped with a 2D LiDAR sensor and an onboard Raspberry Pi3B device. The robot is a differential drive system, its control input $u_t$ is composed by the target linear and angular velocities, which are tracked by a low-level controller. This hardware configuration reflects a prototypical edge robotics scenario, where both compute and memory resources are constrained.}

\rebuttal{We evaluated the proposed algorithm on two robotic navigation tasks. For both settings, the state space, action space, and reward functions are shared.
The state space consists of a downsampled vector of laser scan measurements combined with the linear and angular velocities of the robot:}
\begin{equation}
s_t = \begin{bmatrix} d_1 & d_2 & \ldots & d_n & v_t & \omega_t \end{bmatrix}^T,
\end{equation}
\rebuttal{where \(d_i\) are evenly sampled laser range readings ($n = 16$), $v_t$ and $\omega_t$ are the measured linear and angular velocities. The minimum laser reading \(d_{\min} = \min\{d_1, \ldots, d_n\}\) is used for collision detection and reward shaping.} 

\rebuttal{The action space $a_t \in [-1,1]^2$ is used to control the robot following the control law:}
\begin{align}
    u_t &= B\bigl(M\,a_t + d\bigr) \\&= 
    \begin{bmatrix}
        v_M & 0 \\
        0 & \omega_M
    \end{bmatrix}
    \left(
    \begin{bmatrix}
        1/2 & 0 \\
        0 & 1
    \end{bmatrix}
    a_t + 
    \begin{bmatrix}
        1/2 \\
        0
    \end{bmatrix}
    \right),\\
\end{align}
\rebuttal{where $B \in \mathbb{R}^{2 \times 2}$ is a diagonal matrix containing the maximum linear and angular velocities of the robot, $M \in \mathbb{R}^{2 \times 2}$ is an affine-scaling matrix, and $d \in \mathbb{R}^{2}$ is a translation vector. 
This transformation constrains the robot to move only in the forward direction.}

\rebuttal{The reward function is designed to encourage fast, smooth motion while penalizing proximity to obstacles and excessive angular velocity:}
\begin{align}
r_t &= 
\begin{cases}
3v_t - \left|\frac{\omega_t}{2}\right| - \frac{1}{2}\left(1 - d_{\min} \right), & \text{if } d_{\min} \geq 0.2 \\
-5, & \text{otherwise}
\end{cases}
\end{align}
\rebuttal{The episodes end when $d_{\text{min}}<0.2$ or a maximum time limit is reached. }
\rebuttalnew{This reward encourages a reactive navigation that maximizes the travel distance, measured by integrating the linear velocity, while performing the least possible turns.}

\begin{enumerate}
    \item \rebuttalnew{Corridor Navigation: In this task, the robot operates within a confined rectangular ring-shaped environment, \cref{fig:real-images} (top). The episode time limit is $17$ seconds.}
    
    \item \rebuttalnew{Unstructured navigation: In this task, the robot is deployed within a substantially larger and more complex indoor environment populated with a diverse array of static obstacles, shown in  \cref{fig:real-images} (bottom). The episode time limit is $25$ seconds.}
\end{enumerate}
\rebuttalnew{
At the beginning of the training sequences of both scenarios, the robot is manually placed in a random location. Then it runs autonomously until the maximum time is reached. When $d_{min} < 0.2$, a custom routine repositions the agent using the complete laser scan for orientation, ensuring the robot is parallel to the closest vertical surface and $d_{min} > 0.3$. No external localization system or odometry is employed in these experiments.}
\rebuttal{In both tasks, the goal is to continuously move forward without collisions. The agent receives only real-time LiDAR and speed measurements as observations, and must learn a policy that maximizes the total traveled distance while maintaining safe navigation and avoiding obstacles.}


\begin{figure*}
    \centering
    \includegraphics[width=0.50\linewidth]{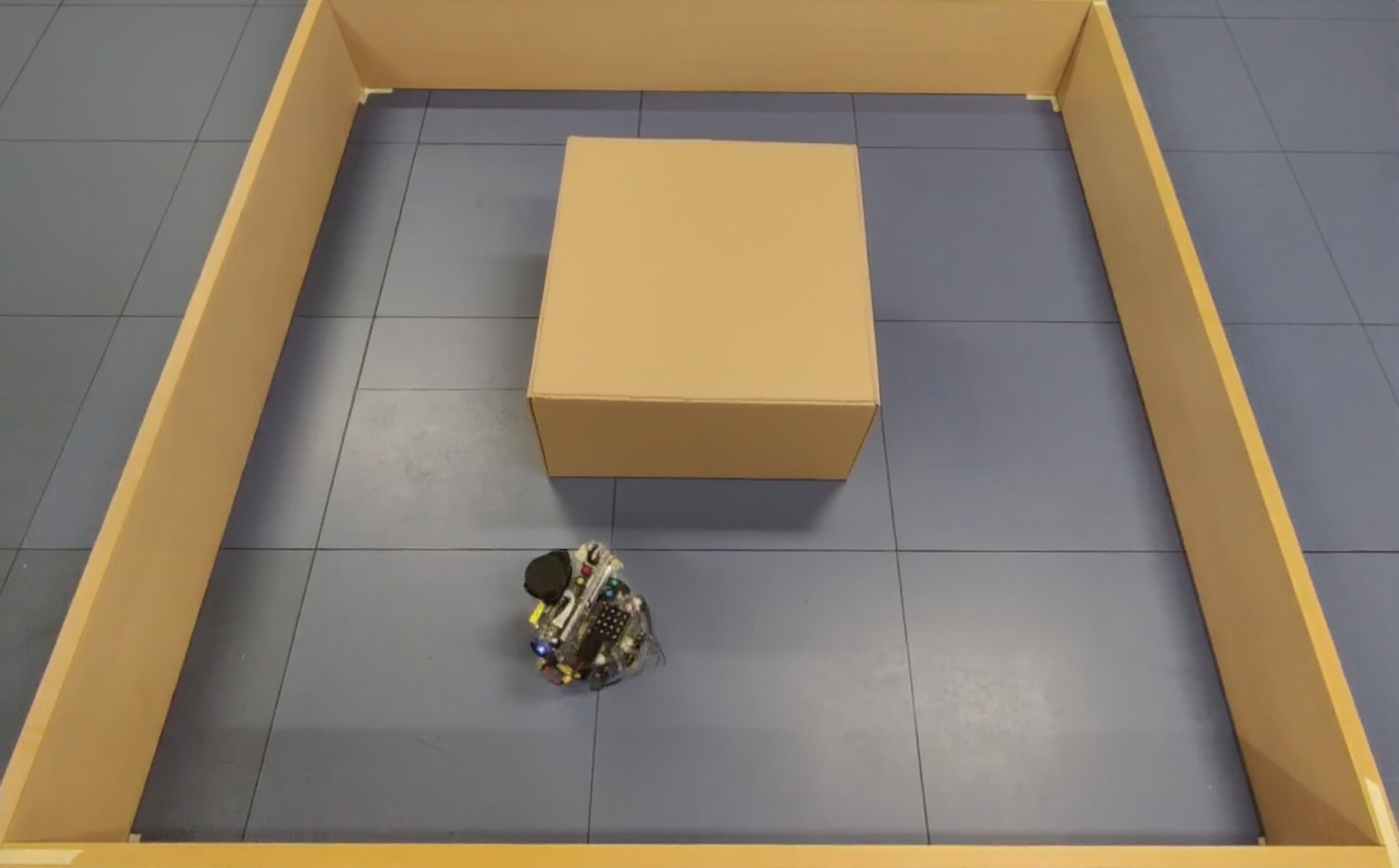}
    \includegraphics[width=0.43\linewidth]{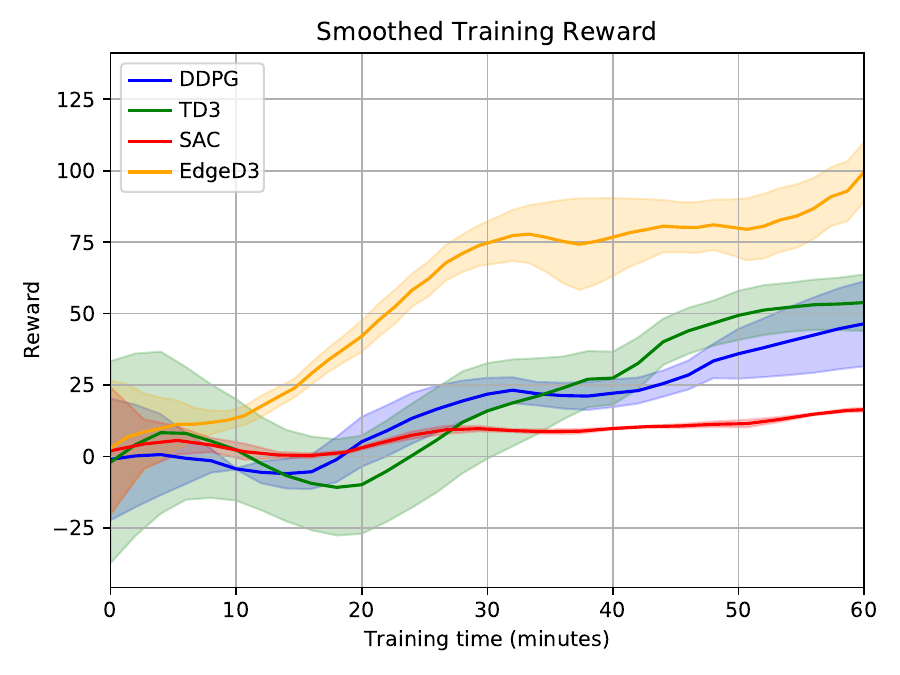}
    \includegraphics[width=0.50\linewidth]{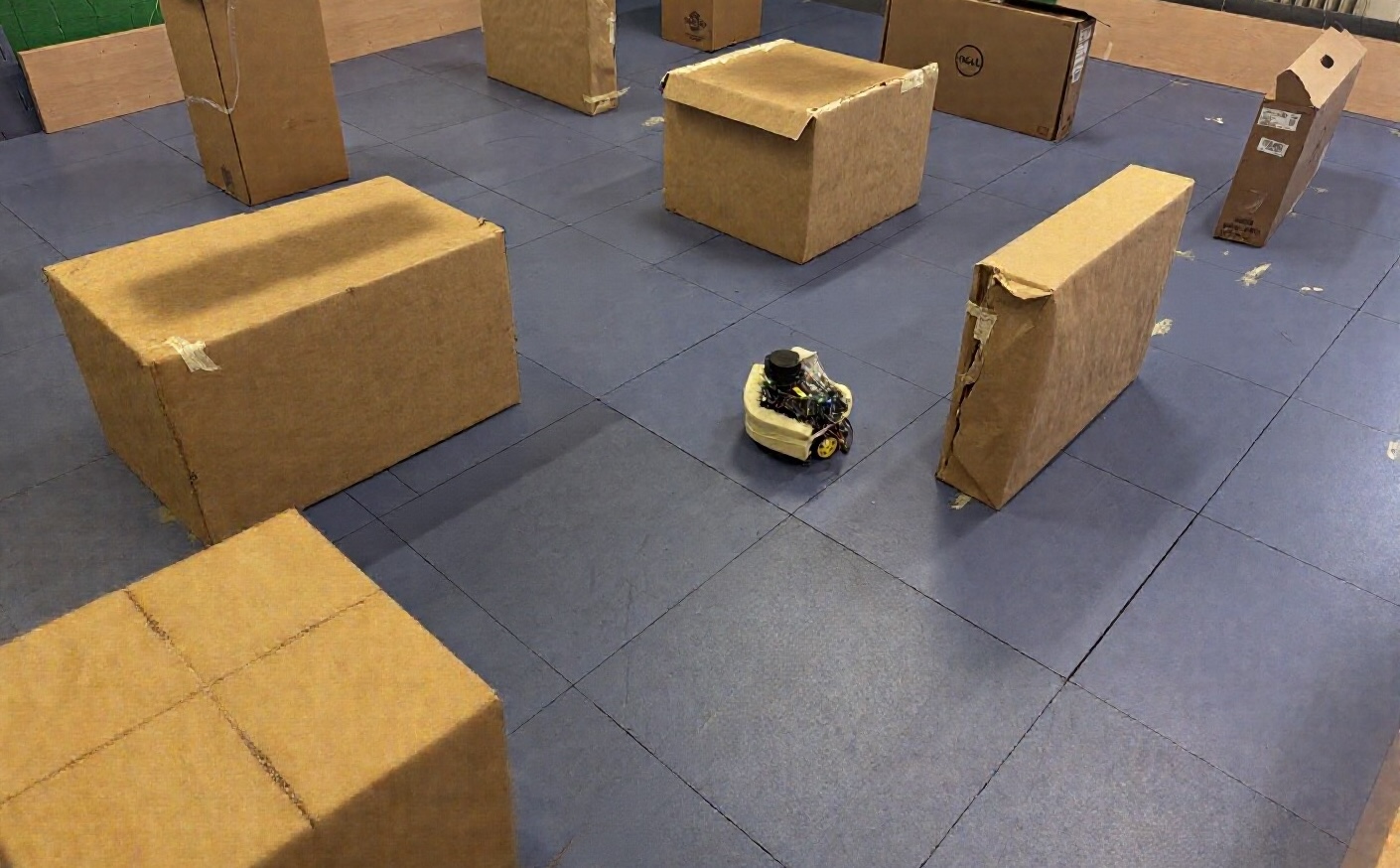}
    \includegraphics[width=0.43\linewidth]{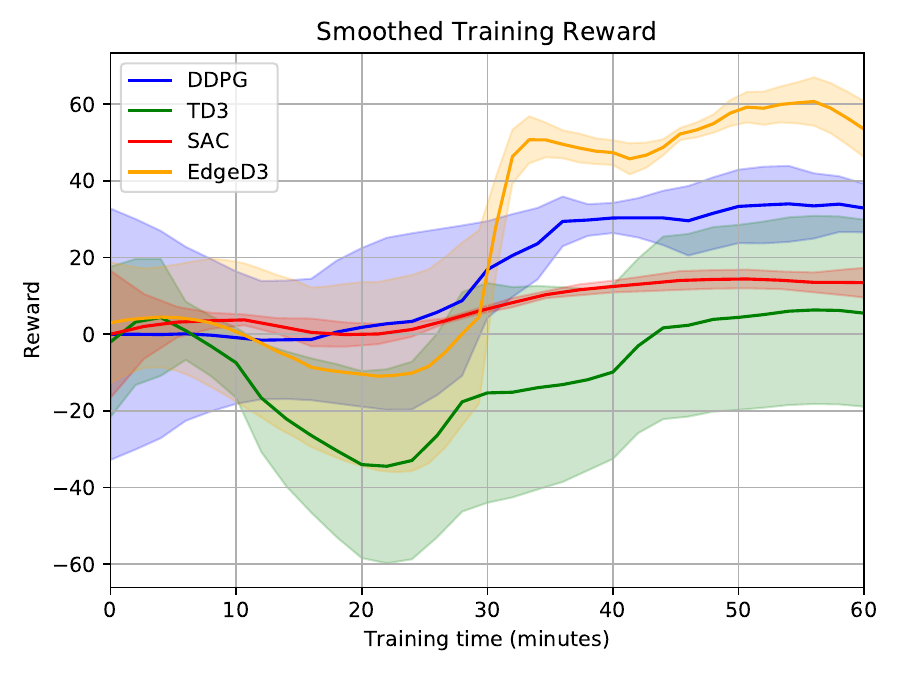}
    \caption{First column shows real-world images of the two tasks, corridor navigation and unstructured navigation. The second column shows the training curves for 1 hour of training for all algorithms, averaged over 5 different seeds.}
    \vspace{-0.5cm}
    \label{fig:real-images}
\end{figure*}

\rebuttal{For both tasks, we compared EdgeD3 to the baselines DDPG, SAC, and TD3 policies, each implemented using the same three-layer, 64-neuron architecture, using a batch size of 128 samples, and all running under identical system constraints. All methods are given 1 wall-clock hour budget for both tasks. Each algorithm was executed for 4 different random seeds.}

\rebuttal{In both scenarios, shown in \cref{fig:real-images}, EdgeD3 obtains strong performance in learning, surpassing all other state-of-the-art methods. In particular, we observe similar relative update rates to the ones measured in the simulations (reported in \cref{table:gpucomp}). Therefore, given the same wall-clock time, EdgeD3 is able to perform more updates than all other approaches, thanks to its lightweight update loop.}

\rebuttal{For reference, on the robot's hardware, the working frequency of the algorithms during training is, respectively $8Hz, 5.9Hz, 5.8Hz, 3.3Hz$ for EdgeD3, TD3, DDPG, and SAC (without delayed updates). Recent algorithms discussed in \Cref{sec:related_work} are not applicable to this setup.}

\rebuttalnew{
A detailed examination of \cref{fig:real-images} reveals distinct learning dynamics for the evaluated algorithms across the two real-world scenarios. In the corridor navigation task (Scenario 1), EdgeD3 demonstrates consistently superior performance from the outset. This is attributable to the relatively structured nature of the environment, where EdgeD3's stabilized target estimation and efficient exploration rapidly yield an effective navigation policy with minimal initial exploration cost.}

\rebuttalnew{Conversely, in the unstructured navigation scenario (Scenario 2), EdgeD3 exhibits suboptimal performance during the initial training phase, with a marked transition to superior performance approximately 30 minutes into training. This phenomenon arises probably due to the increased environmental complexity and variability, which necessitate more extensive exploration and model adaptation before effective generalization can occur. During the initial phase, EdgeD3's conservative target updates, while stabilizing, may delay the exploitation of promising policies as the algorithm accumulates sufficient state-action coverage.}


\begin{table}[h]
\caption{Maximum average return over 10 evaluations across 10 trials,  training a policy for the same amount of wall-clock time. The bolded values represent the best policy learned by each algorithm for each task, and the underlined ones represent the two best.}
\label{table:finalcomp}
\vskip 0.1in
\begin{center}
\begin{small}
\begin{sc}
\begin{tabular}{|c|c|c|c|c|}
\hline
\textbf{Env} & \textbf{ExpD3} & \textbf{DDPG} & \textbf{SAC} & \textbf{TD3}  \\
\hline\hline
Swimmer-V3 & \underline{111.00} & \underline{\textbf{127.82}} & 67.91 & 84.49 \\
Ant-V3 & \underline{\textbf{4350.04}} & 990.55 & 2739.81 & \underline{4208.10} \\
Reacher-V2 & \underline{\textbf{-3.77}} & -4.01 & \underline{-3.79} & -3.84 \\
Hopper-V3 & \underline{\textbf{3388.44}} & 2222.85 & \underline{3148.89} & 2786.22 \\
Walker2D-V3 & \underline{\textbf{3788.07}} & 1601.16 & 2974.40 & \underline{3580.83} \\
Humanoid-V3 & \underline{4331.23} & 1097.72 & 3915.94 & \underline{\textbf{4728.36}} \\
HalfCheetah & \underline{\textbf{10645.8}} & \underline{10309.0} & 8937.3 & 9677.5 \\
\hline
\end{tabular}
\end{sc}
\end{small}
\end{center}
\vspace{-0.5cm}
\end{table}


\section{Conclusions and future work}
\label{sec:conclusions}
Edge computing is always gaining more attention thanks to its ability to allow for scalable deployment and preserving the privacy of the final user, handling the computation directly on-device.
We present Edge Delayed Deep Deterministic Policy Gradient (EdgeD3), which builds on top of Deep Deterministic Policy Gradient (DDPG) \cite{lillicrap2015ddpg}. We introduce a new lightweight, easy-to-implement, highly tunable loss that trades off overestimation and underestimation by exploiting an unbalanced loss, described in \cref{eq:expectile}. Furthermore, we include new tricks to stabilize the training without additional costs. This algorithm aims to be a step towards scalable Deep Reinforcement Learning algorithms for edge computing, thus aiming at minimizing the computational cost and the memory footprint, while not hindering performance.
As done by previous works, such as TD3 and SAC, it achieves better performance by tackling the overestimation bias brought by the temporal difference loss of Q-learning. However, instead of using an ensemble of estimators, EdgeD3 exploits a new Expectile loss to do so while avoiding adding computational burden, thus keeping its property of being edge-friendly.

This new expectile loss, combined with additional blocks proposed by various other works in literature, such as the target smoothing and the delayed policy update, allow us to create a method that is $30\%$ computationally cheaper than the current state-of-the-art methods, preserves a memory footprint on the same level as the original algorithm and $30\%$ smaller than the state of the art, all by preserving the same performances of such more computationally demanding algorithms. \rebuttal{Furthermore, the approach is evaluated on real-world edge scenarios, validating its ability to learn effective policies with minimal computational resources.}

\rebuttal{Potential future research avenues involve investigating other unbalanced losses, such as quantile loss or unbalanced-huber loss, which have the same computational cost as the proposed expectile loss while being more robust to outlier values. Furthermore, the proposed method introduces a hyperparameter that controls the overestimation and the underestimation. Even though an a priori good guess for such hyperparameters exists, there is the possibility to extend the current algorithm with online fine-tuning of such hyperparameters in order to tune it automatically.}

\bibliographystyle{unsrt} 
\bibliography{cas-refs}
\clearpage
\appendix
\subsection{EdgeDDPG pseudocode}
\label{sec:code-expddpg}
\rebuttal{In \cref{alg:expddpg} we report the full pseudocode referring to EdgeDDPG, which is strictly to be used to assess and isolate the performance improvement brought by the newly introduced loss. In particular, the final algorithm still remains  EdgeD3 (\cref{alg:edged3}), which incorporates in EdgeDDPG (\cref{alg:expddpg}) further improvements described in \cref{sec:smoothing}.}

\begin{algorithm}[!h]
   \caption{Edge Deep Deterministic Policy Gradient (EdgeDDPG)}
   \label{alg:expddpg}
\begin{algorithmic}[5]
    \STATE Given $\alpha, \beta, \tau_1, \tau_2$ and $\lambda(t)$
   \STATE Initialize critic $Q_{\theta}$, and actor $\mu_\phi$ networks
   \STATE Initialize target networks $\theta' \leftarrow \theta, \phi' \leftarrow \phi$
   \STATE Initialize replay memory \textbf{$\mathcal{B}$}
   \REPEAT
       \REPEAT
            \STATE Select action with exploration noise $a \sim \mu_{\phi'}(s) + \omega$, $\omega \sim \mathcal{N}(0, \sigma)$ and observe $r$ and  $s'$.
            \STATE $d=\begin{cases} 1 \text{ if } s'\text{ is terminal} \\ 0 \text{ otherwise} \end{cases}$
            \STATE Store $(s,a,r,s',d)$ tuple in $\mathcal{B}$
            \STATE Sample batch of $N$ tuples $(s,a,r,s',d)$ from $\mathcal{B}$
            \STATE $y = r + \gamma Q_{\theta'}(s', \mu_{\phi'}(s')) \cdot (1-d)$ 
            \STATE $\nabla L(\theta) = \nabla_\theta N^{-1} \Sigma L^{\alpha,\beta}(y, Q_{\theta}(s,a))$ [\cref{eq:expectile}]
            \STATE Update $Q_\theta$ via GD using $\nabla L(\theta)$
            \STATE Update $\phi$ by deterministic policy gradient:
            \STATE $\nabla_\phi J(\phi) = \frac{1}{N} \Sigma \nabla_a Q_{\theta}(s,a) |_{a=\mu_\phi(s)} \nabla_\phi \mu_{\phi}(s)$
            \STATE Update target networks:
            \STATE $\theta' \leftarrow \tau_1 \theta + (1-\tau_1) \theta'$
            \STATE $\phi' \leftarrow \tau_2 \phi + (1-\tau_2) \phi'$
            
       \UNTIL{$d$ is false}
       \STATE $\min(\alpha, \beta) \leftarrow \min(\alpha, \beta) + |\alpha - \beta| \cdot \lambda(t) $
   \UNTIL{$t < T$}
\end{algorithmic}
\end{algorithm}

\section{Reproducibility}
\label{sec:app_reproducibility}

\subsection{Practical Considerations for the Selection of $\alpha$ and $\beta$}
\label{sec:alpha-beta-practical}
\rebuttalnew{The hyperparameters $\alpha$ and $\beta$ play pivotal roles in the stability and efficacy of the proposed algorithm. In practice, $\alpha$ and $\beta$ are used to balance the contributions of two competing estimators within the target update mechanism. The regime $\alpha < \beta$ is generally preferred in environments characterized by high stochasticity or when the policy exhibits high variance, as it introduces additional conservatism in target value estimation, thereby mitigating the risk of overestimation bias. This setting is often advantageous in real-world robotics and control scenarios where measurement noise or model uncertainty is prevalent.}

\rebuttalnew{Conversely, the regime $\alpha > \beta$ can be considered in highly deterministic environments or when empirical evidence suggests persistent underestimation of value functions, potentially leading to overly conservative policies. In these settings, a slight relaxation of conservatism may facilitate more aggressive exploration and policy improvement. However, in our experience, such cases are infrequent, and $\alpha \leq \beta$ remains the default configuration for most applications. This regime can be thought of as a loss-induced \textit{optimism in the face of uncertainty}, famous in RL to improve exploration.}

\rebuttalnew{Both $\alpha$ and $\beta$ are treated as tunable hyperparameters and should be selected via standard validation procedures, e.g., grid search or Bayesian optimization, within a plausible range. It is crucial to monitor both the learning stability and the empirical performance during tuning. We emphasize that the primary determinant for $\alpha$ and $\beta$ selection should be the desired trade-off between optimism and conservatism in target value estimation, guided by the stochasticity and complexity of the deployment environment. Default values of $\alpha = 1$ and $\beta = 2$ have demonstrated robust performance across a wide variety of tasks in our evaluations, therefore are highly suggested as a starting point for the tuning, if even needed.}

\subsection{Hyper-parameters and further ablations}
In \cref{table:hp}, we report the hyperparameters used for the simulations. For a fair comparison, the hyper-parameters that could lead to an unfair setting, such as the stepsize, have been kept constant throughout all the methods. For the hyper-parameters that were not common to all of them, we used the one reported in the respective original papers. Regarding the proposed methods, the only parameters that have been varied are the hyperparameters for controlling the trade-off between overestimation and underestimation, formerly $\alpha, \beta$. However, to avoid cherrypicking of such parameters, only good guesses have been used, and are all reported in \cref{fig:all_comparison}. Regarding the decay, we observed little improvement in using both a linear decay and an exponential decay during the execution. Since such decay would be part of non-trivially tunable hyperparameters, and we wanted to keep the method as simple as possible, we decided to use $\lambda(t) = 1$, so no decay has been applied during any of the training reported throughout the paper. Thus, all the curves report learning done with fixed $\alpha, \beta$. Regarding the noise distribution used for the action in the target estimation, we used the clipped Normal distribution introduced in \cite{fujimoto2018td3}.

\begin{table*}[t]
\caption{ 
    List of hyperparameters used for training.}
\label{table:hp}
\vskip 0.15in
\begin{center}
\begin{small}
\begin{sc}
\begin{tabular}{|l|c|c|c|c|c|c|}
\hline
\textbf{Hyper-parameter} & \textbf{TD3} & \textbf{DDPG} & \textbf{EdgeDDPG} & \textbf{EdgeD3} & \textbf{SAC} & \textbf{PPO}\\
\hline\hline
Critic learning-rate& $0.0003$ & $0.0003$ & $0.0003$ & $0.0003$ &  $0.0003$&  $0.0003$ \\
Actor learning-rate& $0.0003$ & $0.0003$ & $0.0003$ & $0.0003$ & $0.0003$& $0.0003$ \\
Optimizer& Adam & Adam & Adam & Adam & Adam& Adam \\
Target Update Rate ($\tau_1, \tau_2$)& $0.005$ & $0.005$ & $0.005$ & $0.005$ &$0.005$ & - \\
Batch-size& $256$ & $256$ & $256$ & $256$ &  $256$ &  $256$ \\
Training iteration per step & $1$ & $1$ & $1$ & $1$ & $1$ & $10$ \\
Discount factor & $0.99$ & $0.99$ & $0.99$ & $0.99$ &  $0.99$ &  $0.99$ \\
Exploration policy & $N(0, 0.2)$ & $N(0, 0.2)$ & $N(0, 0.2)$ & $N(0, 0.2)$ & learnt & - \\
Entropy & - & - & - & - & 0.5 & - \\
Actor heads count & $1$ & $1$ & $1$ & $1$ & $1$ & - \\
Actor update delay ($d$) & $2$ & $1$ & $1$ & $2$ & $1$ & - \\

\hline
\end{tabular}
\end{sc}
\end{small}
\end{center}
\end{table*}

\subsection{Comparison to Actor Critic}
\rebuttal{Another branch of algorithms in reinforcement learning is the family of actor-critic methods. Actor-critic algorithms combine the advantages of value-based and policy-based approaches by maintaining both a parameterized policy function (the actor) and a value function (the critic). This structure enables the actor to update its policy in the direction suggested by the critic, which estimates the expected return. Classical actor-critic methods, however, are often susceptible to issues such as high variance in policy gradient estimates and instability during training. Yet, they offer learning policies with lower memory requirements than DDPG-like algorithms, as they require only 2 networks, the critic and the actor, without target networks. In this section, we thus compare our proposed approach to Proximal Policy Optimization \cite{schulman2017proximal}, which represents the state-of-the-art algorithm in this family, and the most adopted one, thanks to its performance and robustness.}

\rebuttal{Regarding memory consumption, we find that PPO utilizes 18\% less RAM thanks to the fact that no target network is required. Yet, given that the target network is not optimized by gradient descent, the memory consumption required to store it is less than expected (1 out of 3 neural networks). However, PPO, in order to obtain its performance, utilized multiple steps of updates, with an off-policy loss. These multiple updates makes PPO much more expensive than the proposed approach, requiring 35\% more compute for the same training budget. The overall performances are shown in \cref{fig:ppo}. Regarding EdgeD3, the tradeoff $\alpha =1, \beta=2$ has been used for the comparison, as per the final comparison in \cref{fig:benchmarks_edged3}. Regarding PPO, as reported in the original paper, 1 step of updates was done for the value network, and 10 for the actor. Given that such an update count increases the effectiveness of PPO while also increasing computational cost, we did not find a major improvement in increasing over 10 such update count.}

\begin{figure*}[t]
    \centering
    \includegraphics[width=\textwidth]{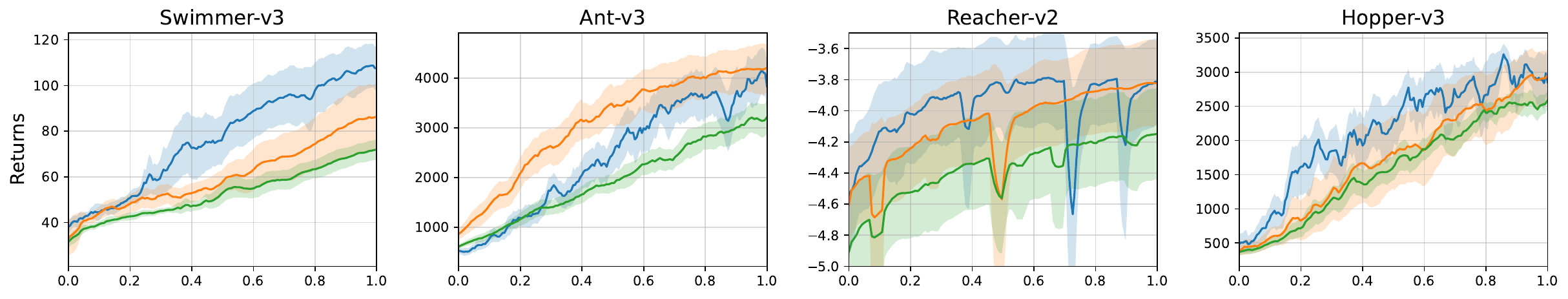}
    
    \includegraphics[width=0.75\textwidth]{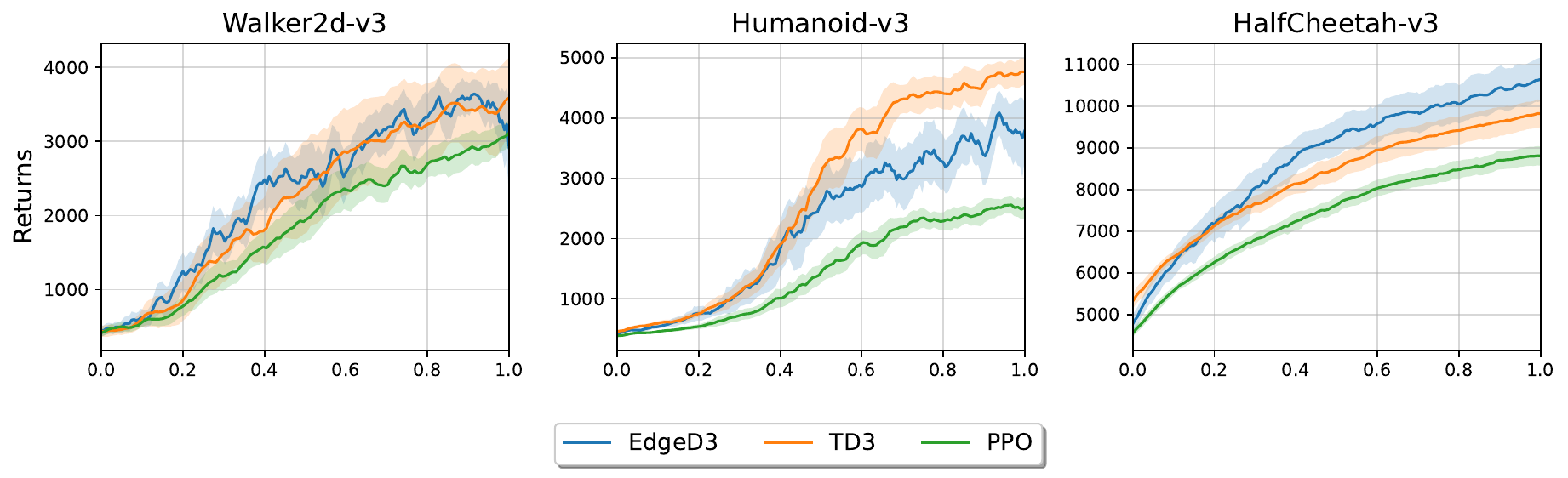}
    \caption{
    Comparing EdgeD3 with PPO and TD3 in continuous control tasks. 
    Plots are from 10 random seeds for simulator and network initializations, smoothed for visualization. Evaluations of Return are performed every 5000 time steps, plots show mean and standard deviation, over 10 episodes.}
    \label{fig:ppo}
\end{figure*}

\section{Ablation over various tradeoffs}
\label{sec:app_ablation_trade_off}
Indeed, the expectile loss allows for very simple and flexible control over the tradeoff between underestimation and overestimation compared to the CDQ mechanism. Thanks to such freedom, we can actually evaluate various values for $\alpha, \beta$ in order to understand the problem we are trying to solve. For this reason, in \cref{fig:all_comparison}, we compare different tradeoffs. 

\begin{figure*}[t]
    \centering
    \includegraphics[width=\textwidth]{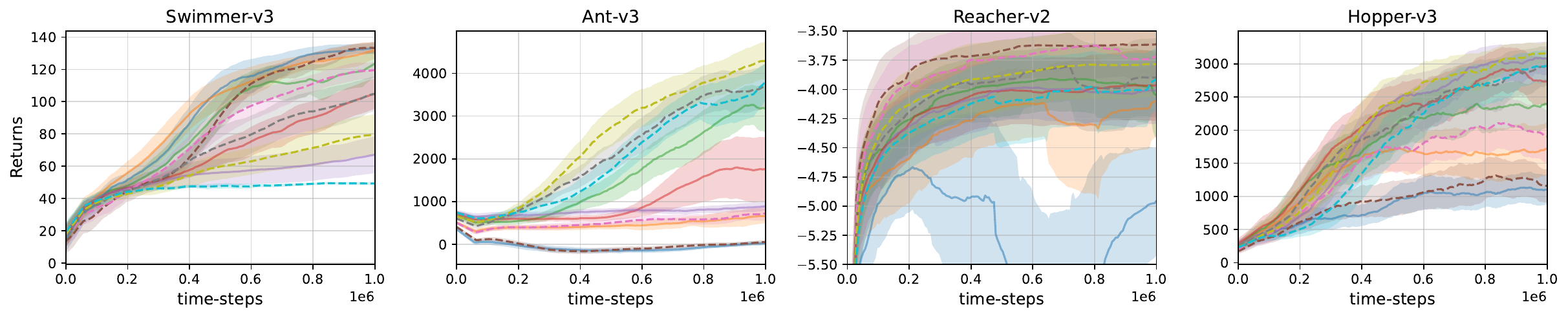}
    
    \includegraphics[width=0.75\textwidth]{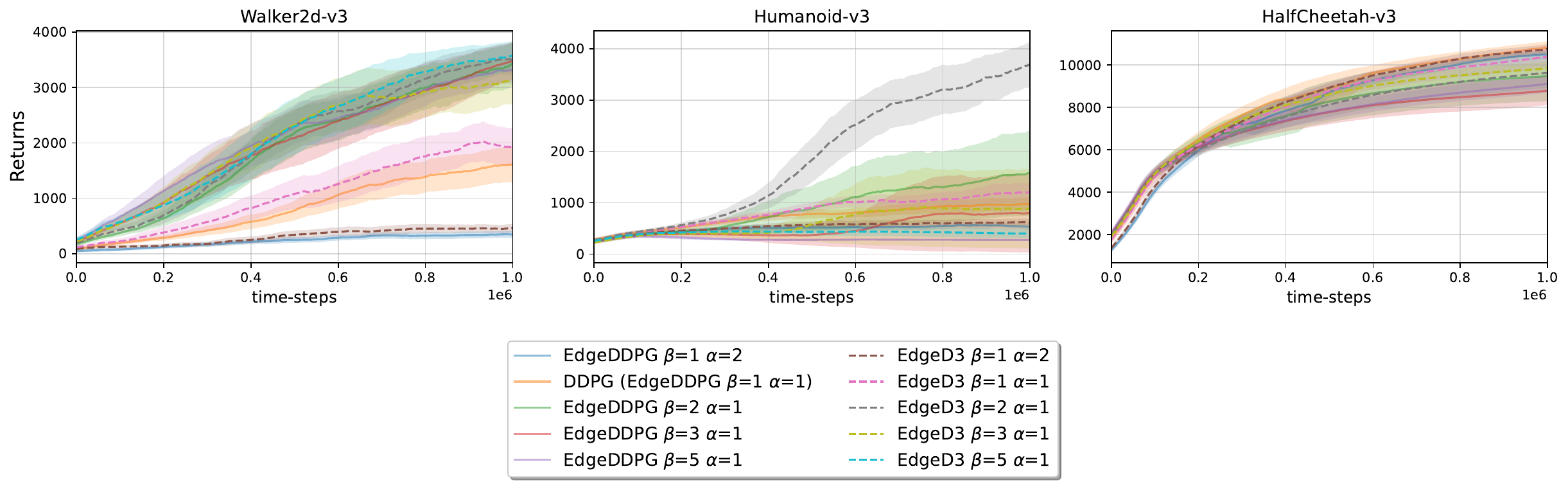}
    \caption{
    Comparing EdgeD3 with EdgeDDPG in continuous control tasks. 
    Plots are from 10 random seeds for simulator and network initializations, smoothed for visualization. Evaluations of Return are performed every 5000 time steps, plots show mean and standard deviation, over 10 episodes.}
    \label{fig:all_comparison}
\end{figure*}

\begin{figure*}[t]
    \centering
    \includegraphics[width=\textwidth]{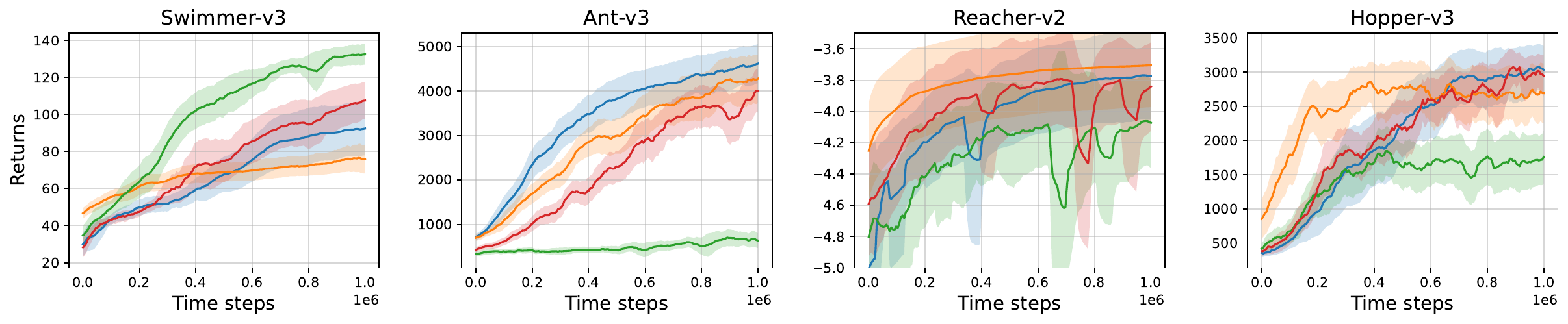}
    
    \includegraphics[width=0.75\textwidth]{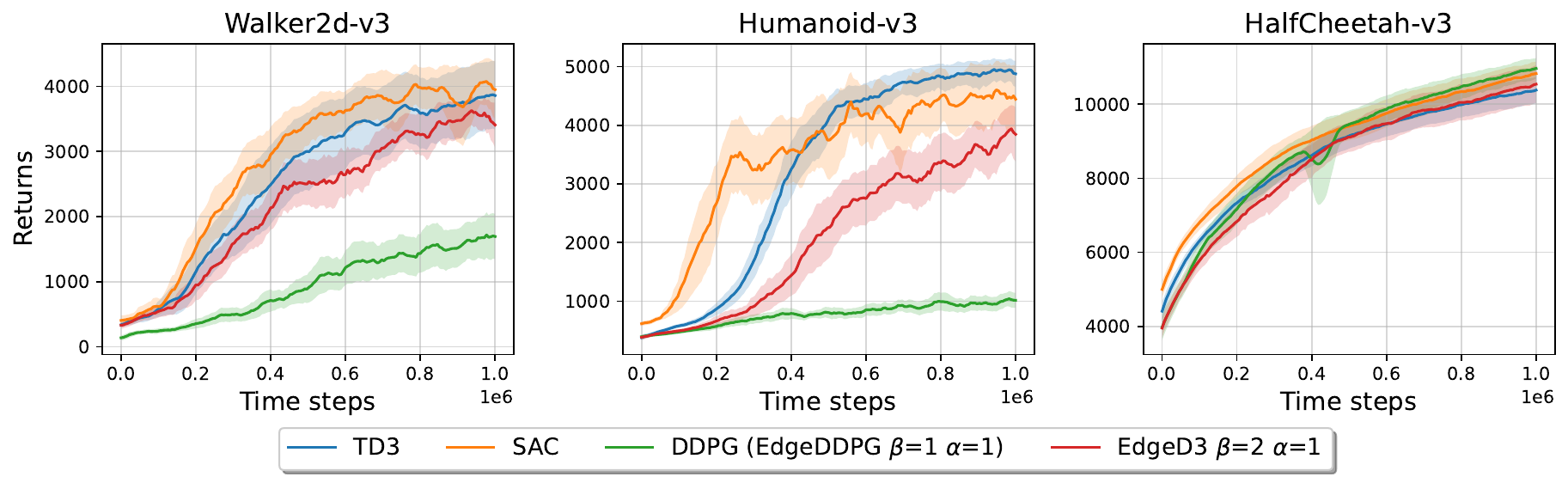}
    \caption{
    Comparing EdgeD3 with baselines in continuous control tasks using environment timesteps. 
    Plots are from 10 random seeds for simulator and network initializations, smoothed for visualization. Evaluations of Return are performed every 5000 time steps, plots show mean and standard deviation, over 10 episodes.}
    \label{fig:comparison_iters}
\end{figure*}

We can observe, particularly in Swimmer, how the idea that underestimation is always a better option than overestimation is definitely not true. Indeed, we can observe a clear positive correlation between how much we prefer overestimation to the final policy performance. If we didn't have an algorithm that allows for such control, we would forced to accept the predefined algorithm performance.\\
Instead, in cases such as Ant, we can see that the trend is definitely the opposite, with the settings that prefer underestimation outperforming the one that prefers underestimation, showing how there is no a priori always-correct choice between the two.\\

\section{Comparison with environment steps}
In \cref{fig:comparison_iters} we report the comparison of the proposed algorithm and the baselines using environment timesteps as a unit of measure. However, we want to emphasize that this might be misleading, as gives no sense of the computational cheapness of the proposed method. Indeed, the aim of this paper is to present a new, computationally, and memory-cheap algorithm suited for edge scenarios. Therefore the focus was on speed and lightness, showing how, taking into consideration these properties, the performance of very established methods such as TD3 and SAC can be reached with much cheaper alternatives.

In \cref{fig:comparison_iters}, we can see how even considering environment timesteps as a unit of measure of progress, EdgeD3 almost always reaches performances very comparable to state-of-the-art methods while taking $25\%$ less time and $30\%$ less memory.

\subsection{Networks architectures}
For TD3, DDPG, EdgeDDPG, \rebuttal{PPO}, and EdgeD3, the actor consists of two fully connected layers of 256 units each with ReLU activations, followed by a fully connected output layer mapping to the action dimension with a hyperbolic tangent activation. For SAC, the actor has the same structure except that the output layer maps to twice the action dimension. For all methods, the critic consists of two fully connected layers of 256 units each with ReLU activations, followed by a single-unit output layer.



\end{document}